\documentclass{article}

\usepackage{xcolor, colortbl} 
\usepackage{wrapfig} 

\usepackage{graphicx}
\usepackage{natbib}
\usepackage{soul}
\setstcolor{red}
\usepackage{nicefrac}
\usepackage{amsmath}
\usepackage{amsthm}
\usepackage{amssymb}
\usepackage{bm} 
\usepackage{bbm} 
\usepackage{mathtools}
\usepackage{enumitem}
\usepackage[ruled, linesnumbered]{algorithm2e}
\usepackage{hyperref}

\usepackage{subcaption}
\usepackage{diagbox}

\usepackage{makecell}
\newtheorem{theorem}{Theorem}
\newtheorem{lemma}{Lemma}
\newtheorem{proposition}{Proposition}

\newtheorem{assumption}{Assumption}
\newtheorem{example}{Example}

\newtheorem{remark}{Remark}
\newtheorem{problem_def}{Problem}

\theoremstyle{definition}
\newtheorem{definition}{Definition}

\newcommand{\lp}{\left (} 
\newcommand{\rp}{\right )} 
\newcommand{\lb}{\left [}
\newcommand{\rb}{\right ]}

\newcommand{\mc}[1]{\mathcal{#1}}
\newcommand{\mbb}[1]{\mathbb{#1}}

\newcommand{\indi}[1]{\mathbbm{1}_{ \{#1\} }}

\newcommand{\reals}{\mathbb{R}} 
\usepackage{blindtext}

\newcommand{\X}{\mathcal{X}}
\newcommand{\Y}{\mathcal{Y}}
\newcommand{\F}{\mathcal{F}}

\DeclareMathOperator*{\argmax}{arg\,max}
\DeclareMathOperator*{\argmin}{arg\,min}

\newcommand{\tbf}[1]{\textbf{#1}}

\makeatletter
\newcommand{\pushright}[1]{\ifmeasuring@#1\else\omit\hfill$\displaystyle#1$\fi\ignorespaces}
\newcommand{\pushleft}[1]{\ifmeasuring@#1\else\omit$\displaystyle#1$\hfill\fi\ignorespaces}
\makeatother

\usepackage[preprint]{neurips_2021}

\usepackage[utf8]{inputenc} 
\usepackage[T1]{fontenc}    
\usepackage{url}            
\usepackage{booktabs}       
\usepackage{amsfonts}       
\usepackage{nicefrac}       
\usepackage{microtype}      
\usepackage{xcolor}         

\title{Adaptive Sampling for Minimax Fair Classification}

\author{%
  Shubhanshu Shekhar \\
  ECE Department, UCSD\\
  \texttt{shshekha@ucsd.edu} \\
   \And
   Greg Fields \\
   ECE Department, UCSD \\
   \texttt{grfields@ucsd.edu} \\
   \And
   Mohammad Ghavamzadeh \\
   Google Research \\
   \texttt{ghavamza@google.com} \\
   \And
   Tara Javidi \\
   ECE Department, UCSD \\
   \texttt{tjavidi@ucsd.edu} \\

}

\begin{document}

\maketitle

\vspace{-0.175in}
\begin{abstract}
\vspace{-0.05in}
Machine learning models trained on uncurated datasets can often end up adversely affecting inputs belonging to underrepresented groups. To address this issue, we consider the problem of adaptively constructing training sets which allow us to learn classifiers that are fair in a {\em minimax} sense. We first propose an adaptive sampling algorithm based on the principle of \emph{optimism}, and derive theoretical bounds on its performance. We also propose heuristic extensions of this algorithm suitable for application to large scale, practical problems. Next, by deriving algorithm independent lower-bounds for a specific class of problems, we show that the performance achieved by our adaptive scheme cannot be improved in general. We then validate the benefits of adaptively constructing training sets via experiments on synthetic tasks with logistic regression classifiers, as well as on several real-world tasks using convolutional neural networks (CNNs).
\end{abstract}

\newcommand{\tj}[1]{\textcolor{green}{(TJ:~#1)}}
\newcommand{\mgh}[1]{\textcolor{orange}{(MGH:~#1)}}
\newcommand{\ssnotes}[1]{\textcolor{blue}{(SS:~#1)}}
\newcommand{\greg}[1]{\textcolor{red}{(G:~#1)}}

\newcommand{\defined}{\coloneqq}
\newenvironment{proofoutline}
{\renewcommand\qedsymbol{}\proof[Proof Outline]}
{\endproof}


\newcommand*{\lpgpucb}{\texttt{LP-GP-UCB}\xspace}
\newcommand*{\heuristic}{\texttt{Heuristic}\xspace}
\newcommand*{\empirical}{\texttt{Empirical}\xspace}
\newcommand*{\epsg}{$\epsilon$-\texttt{greedy}\xspace}
\newcommand*{\toymodel}{\texttt{SyntheticModel1}} 
\newcommand*{\toymodelii}{\texttt{SyntheticModel2}} 
\newcommand{\smallestbudget}{\ttt{SmallestBudget}}
\newcommand{\f}{f}

\newcommand{\ttt}[1]{\texttt{#1}}

\newcommand{\Aopt}{\mc{A}_{\texttt{opt}}}
\newcommand{\Aeps}{\mc{A}_{\epsilon}}

\newcommand{\fairml}{Fair-ML\xspace}
\newcommand*{\loss}{\ell}
\newcommand*{\iid}{\text{i.i.d.}\xspace}

\newcommand{\Mu}{\pmb{\mu}}
\newcommand*{\N}{\mathbb{N}}
\newcommand{\att}{\mathfrak{a}}

\newcommand{\domain}{\mathcal{X}} 

\newcommand*{\Ehatz}[2][t]{\widehat{L}_{#1}(#2)}
\newcommand*{\Ez}[2][z]{L(#1, #2)}

\newcommand*{\Zover}{\mc{Z}_{o}}
\newcommand*{\Zunder}{\mc{Z}_{u}}

\newcommand{\cont}[1]{\mc{C}\lp #1 \rp }

\newcommand{\Tau}{\mathrm{T}}

\newcommand{\Oh}[1]{\mathcal{O}\lp #1 \rp}
\newcommand{\tOh}[1]{\tilde{\mathcal{O}} \lp #1 \rp }

\newcommand{\tPhi}{\widetilde{\Phi}}
\newcommand{\bl}{\;\bullet \;}

\newcommand{\propref}[1]{Proposition~\ref{#1}}
\newcommand{\algoref}[1]{Algorithm~\ref{#1}}
\newcommand{\thmref}[1]{Theorem~\ref{#1}} 
\newcommand{\lemmaref}[1]{Lemma~\ref{#1}}
\newcommand{\defref}[1]{Definition~\ref{#1}} 
\newcommand{\assumpref}[1]{Assumption~\ref{#1}}
\newcommand{\figref}[1]{Figure~\ref{#1}}

\definecolor{mygray}{rgb}{0.90, 0.90, 0.90}

\newcommand*{\Dz}{\mc{D}_{(z)}} 

\vspace{-0.1in}
\section{Introduction}
\vspace{-0.05in}
\label{sec:introduction}
\newlength{\threesubht}
\newsavebox{\threesubbox}
\vspace{-0.7em}
Machine learning~(ML) models are increasingly being applied for automating the decision-making process in several sensitive applications, such as loan approval and employment screening. However, recent work has demonstrated that discriminatory behaviour might get encoded in the model at various stages of the ML pipeline, such as data collection, labelling, feature selection, and training, and as a result, adversely impact members of some protected groups in rather subtle ways~\citep{barocas2016big}. This is why ML researchers have started to introduce a large number of {\em fairness measures} to include the notion of fairness in the design of their algorithms. Some of the important measures of fairness include demographic parity~\citep{zemel2013learning}, equal odds and opportunity~\citep{hardt2016equality,Woodworth17LN}, individual fairness~\citep{dwork2012fairness}, and minimax fairness~\citep{feldman2015certifying}. The {\em minimax} fairness is particularly important in scenarios in which it is necessary to be as close as possible to equality without introducing unnecessary harm~\citep{Ustun19FH}. These scenarios are common in areas such as healthcare and predicting domestic violence. A measure that has been explored to achieve this goal is predictive risk disparity~\citep{feldman2015certifying,Chen18CD,Ustun19FH}. Instead of using the common approach of putting constraints on the norm of discrimination gaps,~\citet{martinez2020minimax} has recently introduced the notion of {\em minimax Pareto fairness}. These are minimax classifiers that are on the Pareto frontier of prediction risk, i.e.,~no decrease in the predictive risk of one group is possible without increasing the risk of another one. 

In this paper, we are primarily interested in the notion of minimax fairness in terms of the predictive risk. However, instead of studying the training phase of the ML pipeline, our focus is on the data-collection stage, motivated by~\citet{Jo19} and~\citet{Holstein18}. In particular, we study the following question: \emph{given a finite sampling budget, how should a learner construct a training set consisting of elements from different protected groups in appropriate proportions to ensure that a classifier trained on this dataset achieves minimax fairness?}

Our work is motivated by the following scenario: suppose we have to learn a ML model for performing a task (e.g.,~loan approval) for inputs belonging to different groups based on protected attributes, such as race or gender. Depending upon the distribution of input-label pairs from different groups, the given task may be statistically harder for some groups. Our goal is to ensure that the eventual ML model has optimal predictive accuracy for the worst-off group. We show that, under certain technical conditions, this results in a model with comparable predictive accuracy over all groups. One approach for this would be to train separate classifiers for each group. However, this is often not possible as the group-membership information may not be available at the deployment time, or it may be forbidden by law to explicitly use the protected characteristics as an input in the prediction process~\citep[\S~1]{lipton2017does}. To ensure having a higher proportion of samples from the \emph{harder} groups without knowing the identity of the hard or easy groups apriori, we consider this problem in an \emph{active setting}, where a learner has to incrementally construct a training set by drawing samples one at a time (or a batch at a time) from different groups. Towards this goal, we propose and analyze an adaptive sampling scheme based on the principle of \emph{optimism}, used in bandits literature (e.g.,~\citealt{auer2002finite}), that detects the \emph{harder} groups and populates the training set with more samples from them in an adaptive manner. We also wish to note that bias in ML has multi-faceted origins and that our work here addresses dataset construction and cannot account for bias introduced by model selection, the underlying data distribution, or other sources as discussed in~\citet{hooker21},~\citet{Suresh19}.  We also endeavor to ensure \emph{minimax} fairness, but in some contexts another notion of fairness, such as those mentioned above, may be more appropriate or equitable.  In general, application of our algorithm is not a guarantee that the resulting model is wholly without bias. 

Our \emph{main contributions} are: \tbf{1)} We first propose an optimistic adaptive sampling strategy, $\Aopt$, for training set construction in Section~\ref{sec:algorithm}. This strategy is suited to smaller problems and admits theoretical guarantees.   We then introduce a heuristic variant of $\Aopt$ in Section~\ref{sec:heuristics} that is more suitable to practical problems involving CNNs. \tbf{2)} We obtain upper bounds on the convergence rate of $\Aopt$, and show its minimax near-optimality by constructing a matching lower bound in Section~\ref{subsec:theoretical}. \tbf{3)} Finally, we   demonstrate the benefits of our algorithm with empirical results on several synthetic and real-world datasets in Section~\ref{sec:empirical}.

{\bf Related Work.} The closest work to ours is by~\citet{abernethy2020adaptive}, where they propose an $\epsilon$-greedy adaptive sampling strategy. They present theoretical analysis under somewhat restrictive assumptions and also empirically demonstrate the benefits of their strategy over some baselines. We describe their results in more detail in Appendix~\ref{subsec:analysis_epsilon_greedy}, and employ the tools we develop to analyze our algorithm to perform a thorough analysis of the excess risk of their strategy under a much less restrictive set of assumptions and to show some necessary conditions on the value of their exploration parameter, $\epsilon$.  We find comparable empirical results for both algorithms given sufficient tuning of their respective hyperparameters, we report some of these results in Section~\ref{sec:empirical} and compare the algorithms in Appendix~\ref{appendix:epsilon_greedy}.  
In another related work,~\citet{anahideh2020fair} propose a fair adaptive sampling strategy that selects points based on a linear combination of model accuracy and fairness measures, and empirically study its performance. These results, however, do not obtain convergence rates of the excess risk of their respective methods and only offer implementations for small-scale datasets. 

The above results study the problem of fair classification in an {\em active} setting and target the {\em data-collection} stage of the ML pipeline. There are also works that take a {\em passive} approach to this problem and focus of the {\em training} phase of the pipeline.~\citet{agarwal2018reductions} design a scheme for learning fair binary classifiers with fairness metrics that can be written as linear constraints involving certain conditional moments. This class of fairness metrics, however, do not contain the minimax fairness measure.~\citet{diana2020convergent} propose a method for constructing (randomized) classifiers for minimax fairness w.r.t.~the empirical loss calculated on the given training set. Similarly,~\citet{martinez2020minimax} derive an algorithm for learning a Pareto optimal minimax fair classifier, under the assumption that the learner has access to the true expected loss functions. Thus, the theoretical guarantees in~\citet{diana2020convergent} and~\citet{martinez2020minimax} hold under the assumption of large training sets. Our work, in contrast, constructs the training set incrementally (active setting) from scratch while carefully taking into account the effects of the finite sample size. Finally, we note that the data-collection strategies proposed in our paper can, in principle, be combined with the training methods presented in~\citet{martinez2020minimax} and~\citet{diana2020convergent} to further guarantee the (minimax) fairness of the resulting classifier. We leave the investigation of this approach for future work.

Besides data-collection and training, there have been studies in the fair ML literature on other aspects of the ML pipeline, such as pre-processing~\citep{celis2020data}, learning feature representations~\citep{zemel2013learning}, post-processing~\citep{hardt2016equality}, and model documentation~\citep{Mitchell18}. We refer the readers to a recent survey by~\citet{caton2020fairness} for more detailed description of these results.

\vspace{-1em}

\section{Problem Formulation}
\label{sec:problem-formulation}
\vspace{-0.5em}

Consider a classification problem with the feature (input) space $\X$, label set $\Y$, and a protected attribute set $\mc{Z} = \{z_1, \ldots, z_m\}$. For any $z \in \mc{Z}$, we use $P_z(x,y)$ as a shorthand for $P_{XY|Z=z}\lp X=x, Y=y \mid Z=z\rp$ to denote the feature-label joint distribution given that the protected feature value is $Z=z$. We also use $\mbb{E}_z[\cdot]$ as a shorthand for the expectation w.r.t.~$P_z$. A (non-randomized) classifier is a mapping $f: \X \mapsto \Y$, which assigns a label $y \in \Y$ to every feature $x \in \X$. For a family of classifiers $\F \subset \Y^\X$, a loss function $\loss: \F \times \X \times \Y \mapsto \reals$, and a mixture distribution over the protected attributes $\pi\in\Delta_m$, we define the \emph{$\pi$-optimal classifier} $f_\pi$ as 
%
\begin{equation}
\label{eq:optimal_classifier_mixture} 
f_\pi \in \argmin_{f \in \mc{F}} \; \mbb{E}_\pi \lb \loss(f, X, Y) \rb := \sum_{z \in \mc{Z}} \pi(z) \; \mbb{E}_z \lb \loss(f, X, Y) \rb. 
\end{equation}
%
When $\pi$ lies on the corners of the simplex $\Delta_{m}$, i.e., $\pi(z)=1$ for a $z \in \mc{Z}$, we use the notation $f_z$ instead of $f_\pi$. For $\pi$ in the interior of $\Delta_m$, it is clear that the risk of $f_\pi$ on group $z$, i.e., $\mbb{E}_z\lb \loss(f_\pi, X, Y)\rb$, must be larger than the best possible classification loss for $P_{XY|Z=z}$. In this paper, our goal is to develop an \emph{active sampling scheme} to find the fair mixture distribution $\pi^*$ in a {\em minimax} sense ({\em minimizing the maximum risk among the groups}), i.e.,
%
\begin{equation}
\label{eq:optimal_mixture}
\pi^* \in \argmin_{\pi \in \Delta_m}\; \max_{z \in \mc{Z}}\; L(z, f_\pi) := \mbb{E}_z \lb \loss(f_\pi, X, Y) \rb. 
\end{equation}
%
The active sampling problem that we study in this paper can be formally defined as follows: 

\begin{problem_def}
\label{problem_def}
Suppose $\mc{F}$ denotes a family of classifiers, $\loss$ a loss function, $n$ is a sampling budget, and $\mc{O}:\mc{Z}\mapsto \X \times \Y$ an oracle that maps any attribute $z \in \mc{Z}$ to a feature-label pair $(X,Y) \sim P_{XY|Z=z}$. The learner designs an adaptive sampling scheme $\mc{A}$ that comprises a sequence of mappings $(A_t)_{t=1}^n$ with $A_t: (\X\times \Y)^{t-1} \mapsto \mc{Z}$ to adaptively query $\mc{O}$ and construct a dataset of size $n$. Let $\pi_n \in \Delta_m$ denote the resulting empirical mixture distribution over $\mc{Z}$, where $\pi_n(z) = N_{z,n}/n$ and $N_{z,n}$ is the number of times that $\mc{A}$ samples from $P_z$ in $n$ rounds. Then, the quality of the resulting dataset is measured by the excess risk, or sub-optimality, of the $\pi_n$-optimal classifier $f_{\pi_n}$, i.e.,
%
\begin{equation}
    \label{eq:risk}
    \mc{R}_n \lp \mc{A} \rp \coloneqq \max_{z \in \mc{Z}} \; L\lp z, f_{\pi_n} \rp  \;- \; \max_{z \in \mc{Z}} \; L(z, f_{\pi^*}) \; ,
\end{equation}
%
where $\pi^*$ is the fair mixture distribution defined by~\eqref{eq:optimal_mixture}. Hence, the goal of the learner is to design a strategy $\mc{A}$ which has a small excess risk $\mc{R}_n(\mathcal A)$. 
\end{problem_def}

Informally, the algorithm should adaptively identify the harder groups $z$ and dedicate a larger portion of the overall budget $n$ to sample from their distributions (see Section~\ref{subsec:motivating_example} for an illustrative example).


\vspace{-0.1in}
\subsection{Properties of the Fair Mixture}
\label{subsec:prop-fair-mixture}
\vspace{-.75em}
As discussed above, our goal is to derive an active sampling scheme that allocates the overall budget $n$ over the attributes $z\in\mathcal Z$ in a similar manner as the {\em unknown} fair mixture $\pi^*$, defined by~\eqref{eq:optimal_mixture}. Thus, it is important to better understand the properties of $\pi^*$ and the $\pi$-optimal classifiers $f_\pi$, defined by~\eqref{eq:optimal_classifier_mixture}. We state three properties of $f_\pi$ and $\pi^*$ in this section. We refer the readers to~\citet{martinez2020minimax} for the definitions of Pareto front and convexity discussed in this section. 

{\bf Property~1.} As discussed in~\citet[Section~4]{martinez2020minimax}, any  $f_\pi$ that solves~\eqref{eq:optimal_classifier_mixture} for a $\pi$ with $\pi(z)>0,\;\forall z\in\mathcal Z$, belongs to the Pareto front $\mathcal P_{\mathcal Z,\mathcal F}$ of the risk functions $\{L(z,f)\}_{z\in\mathcal Z},\;\forall f\in\mathcal F$. 

{\bf Property~2.} We prove in Proposition~\ref{prop:sufficient} (see Appendix~\ref{proof:proposition_sufficient} for the proof) that under the following two assumptions on the conditional distributions $\{P_z\}_{z\in\mathcal Z}$, function class $\mathcal F$, and loss function $\ell$, there exists a unique fair mixture $\pi^*$, whose classifier $f_{\pi^*}$ has equal risk over all attributes $z\in\mathcal Z$. 

\begin{assumption}
\label{assump:ideal_case}
The mapping $\pi \mapsto L(z,f_\pi)$ is continuous for all attributes $z \in \mc{Z}$. Furthermore, if $\pi, \nu \in \Delta_m$ are such that $\pi(z) > \nu(z)$ for an attribute $z\in \mc{Z}$, then $L(z,f_\pi) < L(z,f_\nu)$.
\end{assumption}    

The above assumption indicates that increasing the weight of an attribute $z$ in the mixture distribution must lead to an increase in the performance of the resulting classifier on the distribution $P_z$. 
\begin{assumption}
\label{assump:ideal_case2}
For any two distinct attributes $z, z' \in \mc{Z}$, we must have $L(z,f^*_z) < L(z',f^*_z)$, for any $f^*_z \in \argmin_{f \in \mc{F}}\, L(z,f)$. 
\end{assumption} 
This assumption requires that any optimal classifier corresponding to distribution $P_z$ to have higher risk w.r.t.~the distribution $P_{z'}$ of any other attribute $z'\in\mc{Z}$. Note that Assumption~\ref{assump:ideal_case2} may not always hold, for example, if one attribute is \emph{significantly} easier to classify than another.


\begin{proposition}
\label{prop:sufficient}
Let Assumptions~\ref{assump:ideal_case} and~\ref{assump:ideal_case2} hold for the conditional distributions $\{P_z\}_{z \in \mc{Z}}$, the loss function $\loss$, and the function class $\mc{F}$. Then, there exists a unique $\pi^* \in \Delta_m$ that achieves the optimal value for problem~\eqref{eq:optimal_mixture} and satisfies 
$L(z_1, f_{\pi^*}) = L(z_2, f_{\pi^*}) = \cdots = L(z_m, f_{\pi^*})\defined M^*$.
\end{proposition}
%

In other words, Assumptions~\ref{assump:ideal_case} and~\ref{assump:ideal_case2} define a regime where there exists a fair mixture $\pi^*$, whose classifier $f_{\pi^*}$ achieves complete parity in the classification performance across all attributes. Moreover, Property~1 indicates that $f_{\pi^*}$ also belongs to the Pareto front $\mathcal P_{\mathcal Z,\mathcal F}$. Therefore, under these two assumptions, the equal risk classifier not only belongs to the Pareto front, but it is also the minimax
Pareto fair classifier (Lemma~3.1 in~\citealt{martinez2020minimax}). Note that, in general, without Assumption 1, the classifier that attains equality of risk might have worse performance on all attributes than the minimax Pareto fair classifier (Lemma~3.2 in~\citealt{martinez2020minimax}). 

{\bf Property~3.} We may show that when both the function class $\mathcal F$ and risk functions $\{L(z,\cdot)\}_{z\in\mathcal Z}$ are {\em convex}, $f_{\pi^*}$ is a minimax Pareto fair classifier. As originally derived by~\citet{Geoffrion68PE} and then restated by~\citet[Theorem~4.1]{martinez2020minimax}, under these convexity assumptions, the Pareto front $\mathcal P_{\mathcal Z,\mathcal F}$ is convex and any classifier on $\mathcal P_{\mathcal Z,\mathcal F}$ is a solution to~\eqref{eq:optimal_classifier_mixture}. This together with Property~1 indicates that the classifier $f_{\pi^*}$ corresponding to the fair (minimax optimal) mixture $\pi^*$ is on the Pareto front, and additionally, is a minimax Pareto fair classifier.


\vspace{-0.1in}
\subsection{Synthetic Models}
\label{subsec:motivating_example}
\vspace{-0.075in}

It is illustrative to consider a class of synthetic models which we use as a running example throughout the paper. 

\begin{definition}[\toymodel]
\label{def:toymodel}
Set $\X = \mbb{R}^2$, $\mc{Y} = \{0, 1 \}$, and $\mc{Z} = \{u, v\}$. Each instance of our synthetic model is defined by the set of distributions $P_{XY|Z}$ that satisfy the following: $P_{Y|Z}(Y=1|Z=z) = 0.5$ for both $z\in \mc{Z}$, and $P_{X|Y=y,Z=z}= \mathcal{N}(\mu_{yz}, I_2)$ for all $(y,z) \in \Y \times \mc{Z}$. 
Thus, each instance of this model can be represented by a mean vector $\bm{\mu}=(\mu_{yz}:y\in\mc{Y},\;z\in\mc{Z})$.
\end{definition}

\begin{figure}[b]
\captionsetup[subfigure]{labelformat=empty}
     \centering
    \subcaptionbox{
    \label{fig:synth1}}{\includegraphics[width=0.43\columnwidth]{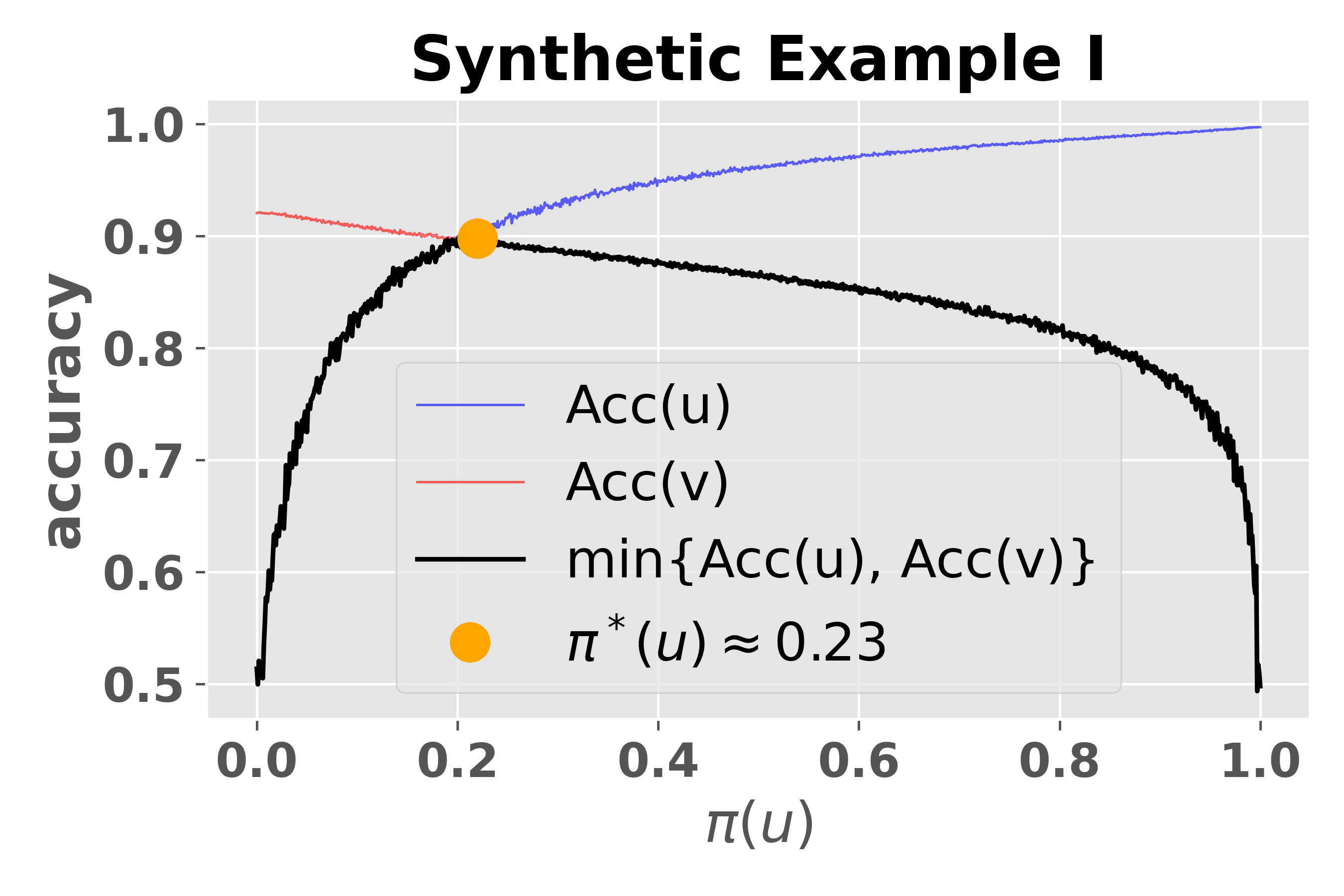}}
    \hfill   
    \subcaptionbox{
    \label{fig:synth2}}{\includegraphics[width=0.43\columnwidth]{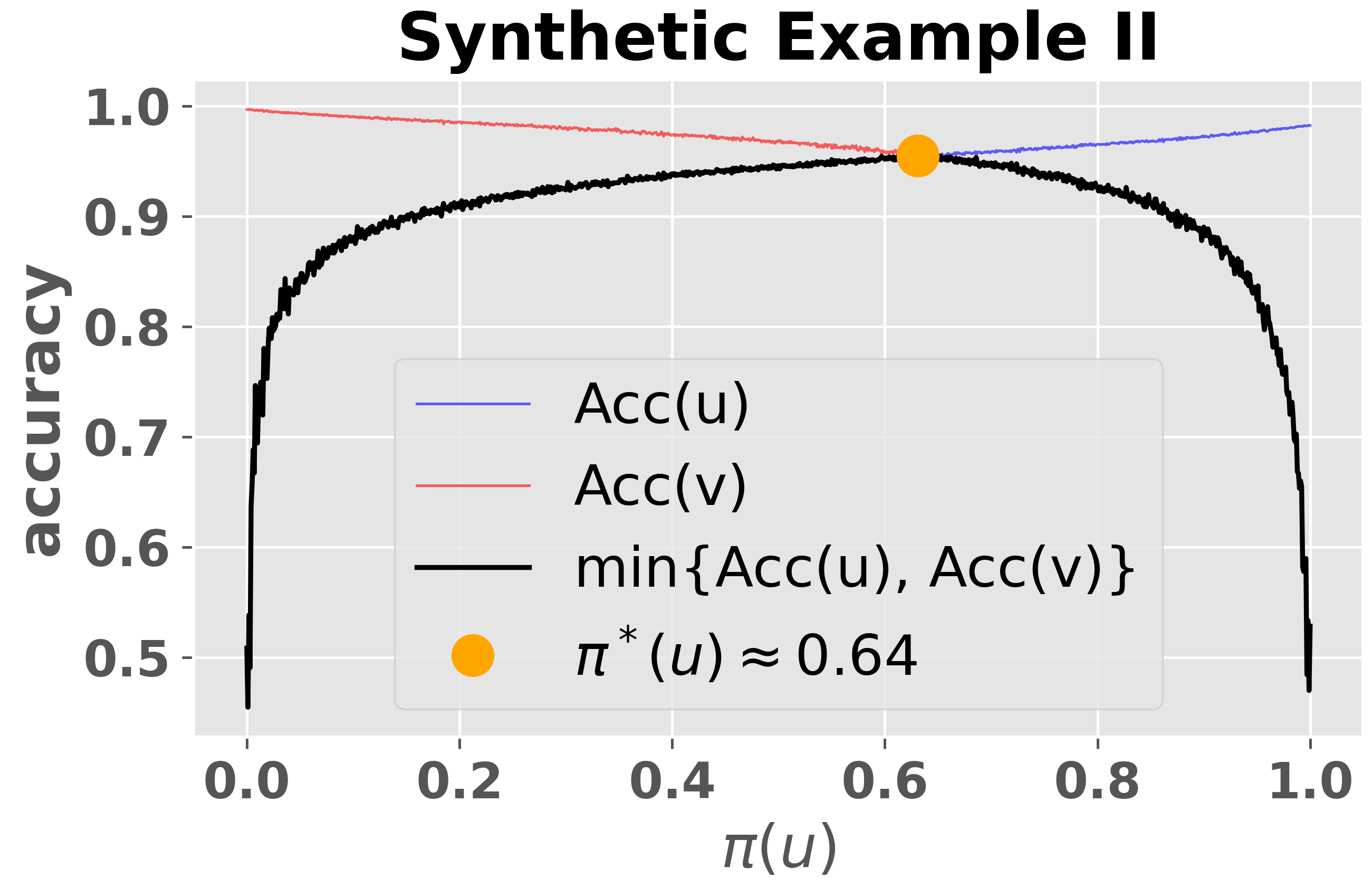}} \\ 
\vspace{-2.6em} 
\caption{The variation of the  minimum prediction accuracy over the two attributes $\mc{Z}=\{u, v\}$ with $\pi(u)$ for Instance I, on the left, and Instance II, on the right, of the \toymodel. }
\label{fig:motivating_example}
\vspace{-.25cm}
\end{figure}

The idea behind this class of models is that the `hardness' of the classification problem for a value of $z\in \mc Z$ depends on the distance between $\mu_{0z}$ and $\mu_{1z}$. The more separated these two mean vectors are, the easier it is to distinguish the labels. Thus, depending on this distance, it is expected that different protected attributes may require different fractions of samples in the training set to achieve comparable test accuracy. To illustrate this, we consider two instances of the \toymodel: {\bf I.} $\mu_{0u} = (-2,2)$, $\mu_{1u} = (2,-2)$, $\mu_{0v} = (-1,-1)$, and $\mu_{1v} = (1,1)$, and {\bf II.} $\mu_{0u} = (-1.5, 1.5)$, $\mu_{1u}\!=(1.5,-1.5)$, $\mu_{0v} = (-2,-2)$, and $\mu_{1v} = (2,2)$. 

For both model instances, we trained a logistic regression classifier over training sets with $1001$ equally spaced values of $\pi(u)$ in $[0,1]$. Figure~\ref{fig:motivating_example} shows the test accuracy of the learned classifier for both attributes~(blue and red curves) as well as their minimum~(black curve)  for different values of $\pi(u)$. Since the pair $(\mu_{0u}, \mu_{1u})$ is better separated than $(\mu_{0v}, \mu_{1v})$ in the first instance, it is easier to classify inputs with $Z=u$ than those with $Z=v$, and thus, it requires fewer training samples from $P_{z=u}$ than $P_{z=v}$ to achieve the same accuracy. This is reflected in Figure~\ref{fig:motivating_example} {\em (left)} that shows the best (min-max) performance is achieved at $\pi(u)\approx 0.23$. An opposite trend is observed in the second instance ({\em right} plot in Figure~\ref{fig:motivating_example}), where the mean vectors corresponding to the attribute $Z=v$ are better separated, and hence, require fewer training samples to achieve the same accuracy.


\section{Optimistic Adaptive Sampling Strategy}
\label{sec:algorithm}

\vspace{-1em}
In this section, we present our {\em optimistic adaptive sampling} algorithm $\Aopt$ for Problem~\ref{problem_def}.  
%
%
Algorithm~\ref{algo:algo1full} contains pseudo-code for $\Aopt$, the algorithm proceeds in the following phases:



\textbf{\textit{Phase~1}} In the initial phase $t \leq m = |\mc{Z}|$, we draw two independent samples from $P_z = P_{XY|Z=z}$, for each attribute $z\in\mc{Z}$, and add one to the training dataset $\mc{D}_t$ and the other one to $\mathcal{D}_z$ (Line~\ref{algoline:init_samples}). We use $\mc{D}_t$ to learn a common (over all $z \in \mc{Z}$) classifier $\hat{f}_t$ via \emph{empirical risk minimization}~(ERM). The independent datasets $\{\mathcal{D}_{z_i}\}_{i=1}^m$ are used for estimating the performance of $\hat{f}_t$ for each $z \in \mc{Z}$. 



\begin{algorithm}[t]
\SetAlgoLined
\SetKwInOut{Input}{Input}\SetKwInOut{Output}{Output}
\SetKwFunction{LocalPoly}{LocalPoly}
\SetKwFunction{Expand}{Expand}
\newcommand\mycommfont[1]{\ttfamily\textcolor{blue}{#1}}
\SetCommentSty{mycommfont}

\KwIn{$n$~(budget), $\F$~(function class), $\ell$~(loss function), $\xi \in (0,1)$~(forced exploration term)}
 \tbf{Initialize:}~ $\mc{D}_0 = \emptyset$; $e_{z,1} = +\infty$ and $\Dz = \emptyset$, for all $z \in \mc{Z}$;
 
\tcc{Draw two independent samples from each $z\in \mc{Z}$}
\For{$t=1,\ldots, m$}
 {
 $z \leftarrow z_t \in \mc{Z}$; \\
$\lp X_t^{(i)}, Y_t^{(i)}\rp_{i=1,2} \sim P_{z}$; $\quad \Dz = \{ (X_t^{(1)}, Y_t^{(1)}) \}$; $\quad \mc{D}_t \leftarrow \mc{D}_{t-1} \cup \{ (X_t^{(2)},Y_t^{(2)})\}$;
 \label{algoline:init_samples}}

$\pi_t \leftarrow \lp \frac{1}{m}, \ldots, \frac{1}{m}\rp$; $\quad \hat{f}_t \in \argmin_{f \in \mc{F}}\; \frac{1}{t} \sum_{(x,y) \in \mc{D}_t}\, \ell(f, x, y)$; $\quad N_{z,t} \leftarrow 1,\;\forall z \in \mc{Z}$; \\
\BlankLine
\For {$t = m+1,  \ldots, n$} 
{
\tcc{ choose the next distribution $P_{XY|Z=z}$ to sample}
\uIf{$\min_{z \in \mc{Z}}\; N_{z,t} < t^\xi$ \label{algoline:00}}
{
$z_t \in \argmin_{z \in \mc{Z}} \; N_{z,t}$ $\qquad$ \tcp{Forced exploration} \label{algoline:forced} 
}
\Else{
$z_t = \argmax_{z \in \mc{Z}}\; U_t(z, \hat{f}_t) $ \label{algoline:choose} \\}

\tcc{Collect data}
$\lp X_t^{(i)}, Y_t^{(i)}\rp_{i=1,2} \sim P_{z_t}$ \label{algoline:observe}

\BlankLine
\tcc{Perform the updates}
Update $\big(e_{z}(N_{z,t})\big)_{z \in \mc{Z}}$, \quad $\mc{D}_{t} \leftarrow \mc{D}_{t-1} \cup \{(X_t^{(1)}, Y_t^{(1)})\}, \quad \mc{D}_{(z_t)} \leftarrow \mc{D}_{(z_t)} \cup \{ (X_t^{(2)}, Y_t^{(2)} \}$. \label{algoline:update1} \\
Update $\pi_t \leftarrow \frac{t-1}{t}\pi_t + \frac{1}{t}\indi{z_t}$  \label{algoline:update2} 

\BlankLine
\tcc{Update the classifier $\hat{f}_t$}
$\hat{f}_t \in \argmin_{f \in\F} \;  \frac{1}{t} \sum_{(x,y) \in \mc{D}_t}\,\ell(f, x, y)$ 
\label{algoline:f_hat} \\

}

\KwOut{$\pi_n$}
 \caption{Optimistic Sampling for Fair Classification~($\Aopt$)}
 \label{algo:algo1full}
\end{algorithm}

\textbf{\textit{Phase~2}} In each round $t > m$, we choose an attribute $z_t$ according to the following \tbf{selection rule}: {\em if} there exists an attribute $z$ whose number of samples in $\mc{D}_t$, denoted by $N_{z_t,t}$, is fewer than $t^\xi$ for some input $\xi\in(0,1)$ (Line~\ref{algoline:00}), we set it as $z_t$ (Line~\ref{algoline:forced}), {\em else} we set $z_t$ as the attribute which has the largest upper confidence bound (UCB) (described below) for the risk of the classifier  $f_{\pi_t}$ (Line~\ref{algoline:choose}). Here $\pi_t$ denote the empirical mixture distribution (over the attributes $z$) of $\mc{D}_t$.



\textbf{\textit{Phase~3}} When the attribute $z_t$ is selected in {\em Phase~2}, the algorithm draws a pair of independent samples from $P_{z_t}$, adds one to $\mc{D}_t$ and the other to $\mc{D}_{z_t}$, and updates $N_{z_t,t}$, $\pi_t$, and the uniform deviation bound $e_{z_t}(N_{z_t,t})$ (described below) (Lines~\ref{algoline:observe} to~\ref{algoline:update2}). The updated dataset $\mc{D}_t$ is then used to learn a new candidate classifier $\hat{f}_t$ (Line~\ref{algoline:f_hat}). {\em Phases~2~and~3} are repeated until the sampling budget is exhausted, i.e.,~$t=n/2$ (note that we sample twice at each round).

\tbf{\textit{Calculating UCB.}} To construct the UCB, we introduce two additional assumptions:

\begin{assumption}
\label{assump:ideal_case3}
There exist positive constants $\epsilon_0, C >0$ such that for any function $f \in \mc{F}$ and any $\pi \in \Delta_m$, with $\pi(z)>0,\;\forall z \in \mc{Z}$, if we have $\mbb{E}_\pi \lb \loss \lp f, X, Y \rp \rb \leq \mbb{E}_\pi \lb \loss \lp f_\pi, X, Y \rp \rb + \epsilon$, for some $0 < \epsilon \leq \epsilon_0$, then $ \pi(z) \times \lvert \mbb{E}_z \lb \loss \lp f, X, Y \rp \rb - \mbb{E}_z \lb \loss \lp f_\pi, X, Y \rp \rb \rvert \leq 2C \epsilon,\;\forall z \in \mc{Z}$.
\end{assumption}
 
This assumption requires that whenever for a $f$ and a $\pi$, $\sum_{z} \pi(z)\big(\mbb{E}_z \lb \loss(f, X, Y) \rb - \mbb{E}_z \lb \loss(f_\pi, X, Y) \rb\big)$ is small, the \emph{absolute value} of the individual components of the sum, i.e., $\pi(z)\lvert \mbb{E}_z \lb \loss(f, X, Y) \rb - \mbb{E}_z \lb \loss(f_\pi, X, Y) \rb \rvert$, cannot be too large. Note that this does not exclude the chance that $\lvert \mbb{E}_z \lb \loss(f, X, Y) \rb - \mbb{E}_z \lb \loss(f_\pi, X, Y) \rb \rvert$ is large; for instance, when $\pi(z)$ is small for some $z$. A sufficient condition for this assumption to hold is that the mapping  $\pi \mapsto f_{\pi}$ is injective, and any other local optimizer is at least $\epsilon_0$ sub-optimal. 

\begin{assumption}
\label{assump:ideal_case4}
Let $\delta \in [0,1]$ be a confidence parameter.  For each $z\in \mc{Z}$, there exists a monotonically non-increasing sequence, $\{e_z(N,\mc{F},\delta):N \geq 1\}$, with $\lim_{N\to\infty}e_z(N,\mc{F},\delta)=0$, such that the following event holds with probability at least $1-\delta/2$:
\vspace{-0.1in}
\begin{small}
\begin{align*} 
\Omega = \bigcap_{z \in \mc{Z}} \bigcap_{N=1}^\infty \Big\{\sup_{f \in \mc{F}}\, \big\vert \widehat{L}_N(z, f) - L(z, f) \big\vert \leq e_{z}(N, \mc{F}, \delta)
\Big\},
\end{align*} 
\end{small}
\vspace{-0.1in}

where $\widehat{L}_N(z,f):=\frac{1}{N}\sum_{i=1}^N \ell\big(f,X_z^{(i)},Y_z^{(i)}\big),\ \forall z\in\mc{Z}$, with $\big(X_z^{(i)}, Y_z^{(i)}\big)_{i=1}^N$ being an i.i.d. sequence of input-label pairs from $P_z$. 
\end{assumption}

This assumption is a standard uniform convergence requirement which is satisfied by many commonly used families of classifiers, several examples are detailed below in Remark~\ref{remark:instantiate}.  In what follows, we will drop the $\mc{F}$ and $\delta$ dependence and refer to $e_z(N,\mc{F},\delta)$ as $e_z(N)$ for all $z\in\mc{Z}$ and $N\geq 1$.  Given an appropriate sequence of $e_z(N)$, we construct the UCB for the risk function $L(z, f_{\pi_t})$, defined by Equation~\eqref{eq:optimal_mixture}, with the following expression:
%
\begin{small}
\begin{equation}
U_t(z,\hat{f}_t) \coloneqq \frac{1}{|\mc{D}_z|} \sum_{(x,y) \in \mc{D}_z} \!\!\!\!\loss(\hat{f}_t, x, y) \;+\; e_{z}\lp N_{z,t}\rp
 +  \frac{2C}{\pi_t(z)}\sum_{z' \in \mc{Z}} \pi_t(z') e_{z'}(N_{z',t})\;.
    \label{eq:ucb}
\end{equation}
\end{small} 
\vspace{-0.1in}

Where $C$ is the constant described in Assumption~\ref{assump:ideal_case3}.  The first term on the RHS is the empirical loss of the learned classifier at round $t$ on attribute $Z=z$, so by Assumption~\ref{assump:ideal_case4} the first two terms of the RHS of Equation~\ref{eq:ucb} give a high probability upper bound on $L(z,\hat{f}_t)$, the expected loss of $\hat{f}_t$ conditioned on $Z=z$.  The third term then provides an upper bound on the difference $|L(z,f_{\pi_t})-L(z,\hat{f}_t)|$ and so   altogether this provides the desired UCB on $L(z,f_{\pi_t})$, the expected loss of the $\pi_t$ optimal classifier on attribute $Z=z$. The form of the third term is due to Assumption~\ref{assump:ideal_case3} and is discussed in detail in Appendix~\ref{proof:excess_risk1}.
\vspace{-.5em}
\subsection{Theoretical Analysis}
\label{subsec:theoretical}
\vspace{-.75em}
In this section, we derive an upper-bound on the excess risk $\mathcal{R}_n(\Aopt)$ of \algoref{algo:algo1full}. We also show that the performance achieved this algorithm cannot in general be improved, by obtaining an algorithm independent lower-bound on the excess risk for a particular class of problems.

\tbf{Upper Bound.} We begin by obtaining an upper bound on the convergence rate of the excess risk of $\Aopt$, the proof of this result is in Appendix~\ref{proof:excess_risk1}.

\begin{theorem}
\label{theorem:regret_simple}
Let Assumptions~\ref{assump:ideal_case}-3~ hold and define $\pi_{\min} \coloneqq \min_{z \in\mc{Z}}\pi^*(z)$. Fix any $A$ such that $\pi_{\min}/2 \leq A < \pi_{\min}$. Suppose the query budget $n$ is sufficiently large, as defined in Equation~\ref{eq:ni} in Appendix~\ref{proof:excess_risk1}. Then, with probability $1-\delta$, the excess risk of \algoref{algo:algo1full} can be upper-bounded as 
%
\begin{equation}
\label{eq:excess-risk-Aopt}
\mathcal{R}_n(\Aopt) = \max_{z \in \mc{Z}} \; L(z, f_{\pi_n}) - M^* = \mc{O} \Big(\frac{ |\mc{Z}|C}{\pi_{\min}} \;\max_{z \in \mc{Z}} \; e_{z }(N_A)\Big),
\end{equation}
%
where $M^* = L(z, f_{\pi^*}),\;\forall z \in \mc{Z}$ (see Proposition~\ref{prop:sufficient}) and $N_A = n \pi^2_{\min}/(2\pi_{\min}-A)$. 
\end{theorem}

\begin{remark}
\label{remark:instantiate}
 The uniform deviation bounds, $e_z(N)$, defined in Assumption~\ref{assump:ideal_case4} and the bound on the excess risk of \algoref{algo:algo1full} (see Equation~\ref{eq:excess-risk-Aopt}) can be instantiated for several commonly used classifiers (function classes $\mathcal F$) to obtain an explicit convergence rate in $n$. If $\mc{F}$ has a finite VC-dimension, $d_{VC}$, then a suitable deviation bound is $e_z(N) = {\small 2 \sqrt{ \lp 2d_{VC} \log{\lp \nicefrac{2eN}{d_{VC}}\rp}+2\log{\lp \nicefrac{2N^2\pi^2|\mc{Z}|}{3\delta} \rp} \rp / N}}$. And if $\mc{F}$ has Rademacher complexity $\mathfrak{R}_n$, we can choose $e_z(N) = {\small 2\mathfrak{R}_N + \sqrt{\log{\lp \nicefrac{N^2\pi^2|\mc{Z}}{3\delta}\rp / N }}}$.  Furthermore, with these uniform deviation bounds and if $\mathfrak{R}_n$ is  $\mc{O}\lp n^{-\alpha} \rp$, for some $\alpha>0$, then we obtain excess risk bounds of   
\begin{small}
\begin{equation*}
\label{eq:excess-risk-instantiation}
\mathcal{R}_n(\Aopt) = \mc{O} \Big(\frac{ |\mc{Z}|C}{\pi_{\min}^{1.5}} \; \sqrt{\frac{d_{VC}}{n}}\Big), \qquad 
\mathcal{R}_n(\Aopt) = \mc{O} \Big(\frac{ |\mc{Z}|C}{\pi_{\min}^{1+\alpha}n^\alpha} \Big),
\end{equation*}
\end{small}
 
 for the VC-dimension and Rademacher cases, respectively.  These conditions on $\mathcal F$ are satisfied by several commonly used classifiers, such as linear, SVM, and multi-layer perceptron. 
\end{remark}

\begin{remark}[Analysis of \epsg]
\label{remark:eps-greedy}
\citet{abernethy2020adaptive} only analyze the greedy version (i.e.,~$\epsilon=0$) of their \epsg sampling strategy with $|\mathcal Z|=2$, and show that at time $n$, either the excess risk is of $\mc{O} \big(\max_{z \in \mc{Z}} \sqrt{ \nicefrac{ 2 d_{VC}\lp \loss \circ \mc{F} \rp \log (2/\delta)}{N_{z,n}}}\big)$, or the algorithm draws a sample from the attribute with the largest loss. However, due to the greedy nature of the algorithm they analyzed, there are no guarantees that $N_{z,n} = \Omega(n)$, and thus, in the worst case the above excess risk bound is $\mc{O}(1)$. We show in Appendix~\ref{appendix:epsilon_greedy} how the techniques we developed for the analysis of \algoref{algo:algo1full} can be suitably employed to study their \epsg strategy. In particular, we obtain sufficient conditions on $\epsilon$ under which the excess risk of \epsg converges to zero, and the rate at which this convergence occurs.
\end{remark}

{\bf Lower Bound.} Let $\mc{Q}=(\bm{\mu}, \mc{F}, \loss_{01})$ denote the class of problems where $\bm{\mu} \in \mc{M}$ is an instance of the \texttt{SyntheticModelI} described in Section~\ref{subsec:motivating_example}, $\mc{F}$ is the class of linear classifiers in two dimensions, and $\loss_{01}$ is the $0-1$ loss. For this function class $\mc{F}$, $e_{z}(N) = \mc{O} ( \sqrt{\log(N)/N})$ for $z \in \mc{Z} = \{u, v\}$, which implies that the excess risk achieved by both $\Aopt$ and \epsg strategies is of $\mc{O}\lp \sqrt{\log(n)/n} \rp$. We prove in Appendix~\ref{appendix:proof_lower} that this convergence rate (in terms of $n$) cannot in general be improved by showing that $\max_{Q \in \mc{Q}}\; \mbb{E}_Q\lb \mc{R}_n \lp \mc{A} \rp  \rb = \Omega \lp 1/\sqrt{n} \rp$.

\vspace{-1em}
 


\section{Heuristic Extensions}
\label{sec:heuristics}
\vspace{-1em}
We show in Appendix~\ref{appendix:parameterC} that if the uniform deviation bounds, $e_z(N)$, can be chosen to decrease to 0 sufficiently quickly as $N$ increases, then we can omit the $C$ dependent third term in Equation~\eqref{eq:ucb} and still attain the same regret bounds given in Theorem~\ref{theorem:regret_simple}. The resulting two-term UCB is then only the high probability upper bound on $L(z,\hat{f}_t)$.  We will use this UCB as the basis for several practical modifications to our optimistic adaptive sampling strategy.  First we note that in UCB based algorithms the confidence bounds necessary to attain theoretical results often do not produce optimal empirical results.  So, following standard practice, we introduce a hyperparameter, $c_0$ in Equation~\eqref{eq:UCB} below, which we can tune to optimize the exploration/exploitation trade-off.  

Practical applications of our strategy run into two further challenges.  As described in Section~\ref{sec:algorithm}, $\Aopt$ requires re-training the classifier in every iteration. While this can be implemented in small problems, it becomes infeasible for problems involving large models, such as CNNs. Also, as noted in Section~\ref{subsec:prop-fair-mixture}, for data where one attribute is significantly easier than the other, Assumption 2 may not hold. In this case it may not be beneficial to continue to sample from the attribute with the largest loss.  To address these issues we present a heuristic variant of \algoref{algo:algo1full}. 

For the first challenge we make the following modifications to Algorithm~\ref{algo:algo1full}.  \tbf{1)} We expand Phase 1 to $n_0$ rounds, \tbf{2)} at each subsequent round we draw two batches of size $b_0$ from the chosen attribute, \tbf{3)} instead of re-training from scratch each round, we update the previous model, $\hat{f}_{t-1}$, and \tbf{4)} instead of training to convergence, we perform one gradient step over the entire accumulated training set, $\mc{D}_t$.  

To address the second challenge we modify the UCB by adding a term based on the Mann-Kendall statistic~\citep{mann45}~\citep{kendall48} which, for time series data $(X_i)_{i=1}^n$, is given by {\small $ S = \sum_{i=1}^{n-1} \sum_{j=i+1}^n \operatorname{sgn}{(X_j-X_i)}$ }. This is designed to identify monotonic upwards or downward trends in time series.  We calculate the statistic for each attribute, and denote them $S_{z_t}$, for the accuracy of the classifier, $\hat{f}_t$, on the attribute's validation set, $\mc{D}_{z_t}$, over time.  Intuitively, this is tracking whether training on additional samples from an attribute is improving the classifier's accuracy on that attribute. We incorporate the statistic into the UCB as follows:

    
    
   

\vspace{-1em}
\begin{small}
\begin{equation}
\label{eq:UCB}
\widetilde{U}_t(z, \hat{f}_t) \defined \frac{1}{|\mathcal{D}_z|}\sum_{(x,y) \in \mathcal{D}_z} \!\!\!\!\loss(\hat{f}_t, x, y)\;+\, \frac{c_0}{\sqrt{N_{z,t}}}\;+\, c_1 \frac{S_{z_t}}{\sqrt{\operatorname{var(S_{z_t})}}}, 
\end{equation}
\end{small}
to incentivize the algorithm to sample attributes on which accuracy is increasing. $c_1$ is a free parameter controlling the importance of per-attribute loss trends. We note that the general ability to modify the UCB in this manner is a strength of our algorithm, as it allows for a great deal of interpretability and adaptability in practical applications.  \\
The parameters $c_0$, $c_1$, $n_0$, and $b_0$ are chosen based on the specific problem and available computational resources. We describe our selection of these terms Appendix~\ref{appendix:experiments}.  In the experiments we will refer to the variant of this algorithm with $c_1=0$ as $\Aopt$, and the variant with $c_1>0$ as $\Aopt +$.
\vspace{-1em} 

\section{Empirical Results}
\label{sec:empirical}
\vspace{-1em}

\begin{figure}[t]
\captionsetup[subfigure]{labelformat=empty}
     \centering
    \subcaptionbox{
    \label{fig:toyexpt1}}{\includegraphics[width=0.4\columnwidth]{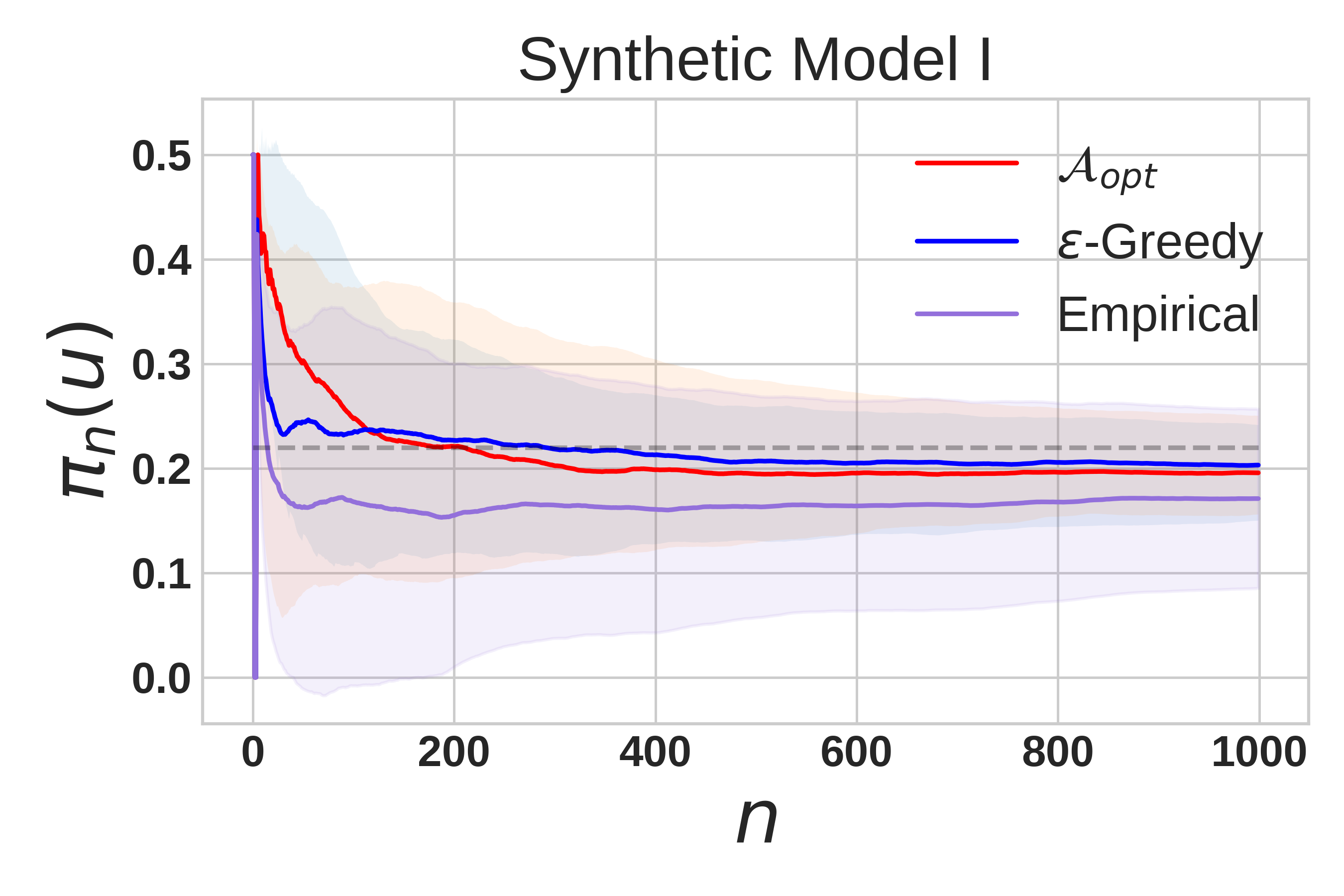}}
    \hfill   
\subcaptionbox{
\label{fig:toyexpt2}}{\includegraphics[width=0.4\columnwidth]{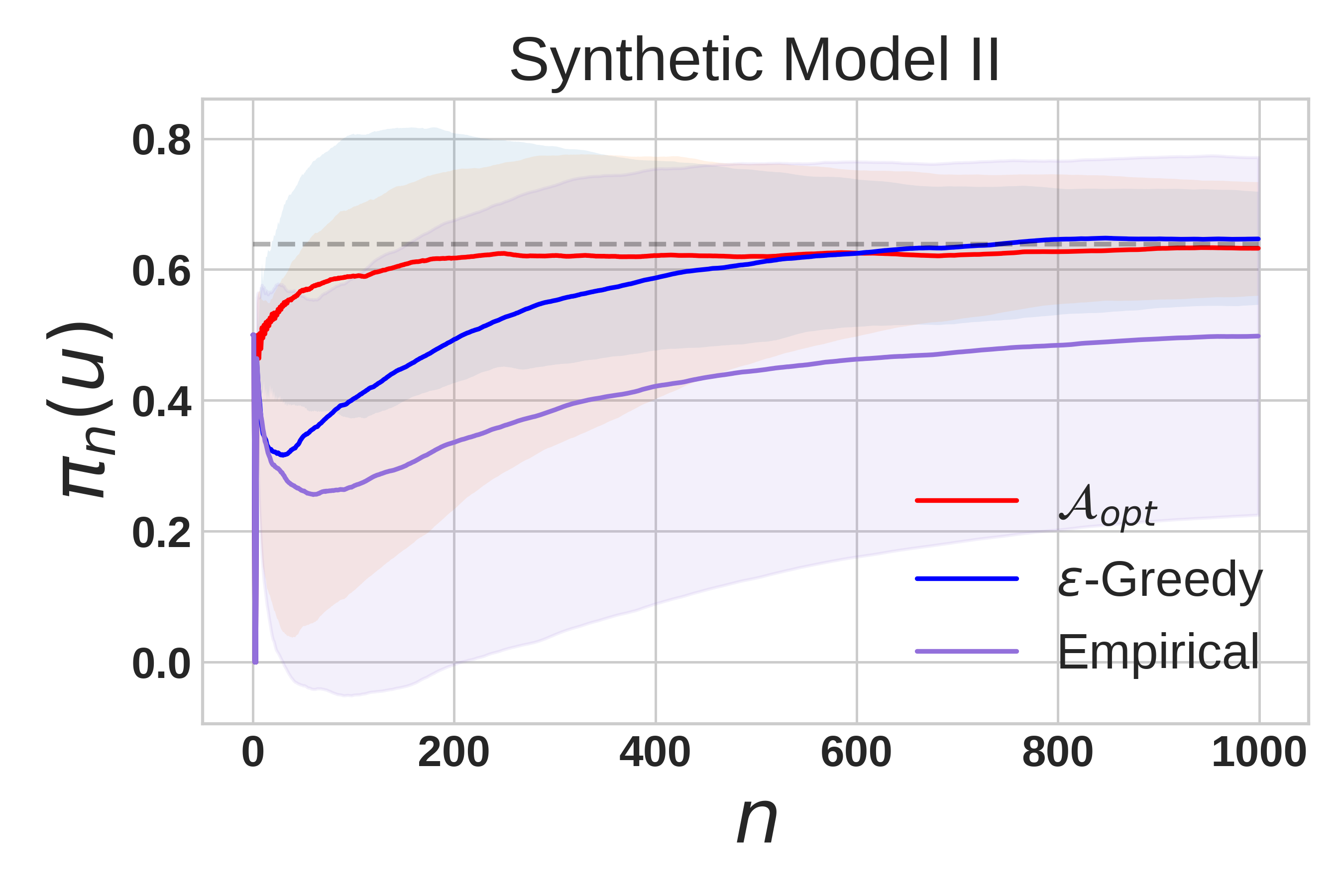}} \\ 
\vspace{-2.5em} 
\caption{The figure shows the convergence of $\pi_n(u)$ for the three algorithms $\Aopt$, \epsg~(with $\epsilon=0.1$) and \empirical (i.e.,~\epsg with $\epsilon=0$), to the optimal value $\pi^*(u)$ for the two instances of the $\toymodel$ introduced in Section~\ref{subsec:motivating_example}. 
}
\label{fig:experiment0}
\vspace{-1.2em}
\end{figure}

\begin{figure}[b]
     \centering
    \includegraphics[width=0.45\columnwidth]{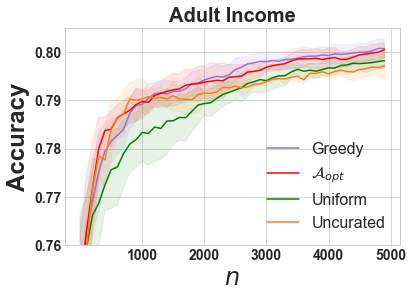}
\caption{Minimum test accuracy over all attributes for {\em Adult}  dataset as a function of the sampling budget $n$, averaged over 10 trials for {\em Adult}}
\label{fig:LR_datasets}
\end{figure}

We evaluate the performance of our proposed active sampling algorithm $\Aopt$ on both synthetic and real datasets, and compare it with the following baselines: \tbf{1)} \epsg scheme of~\citet{abernethy2020adaptive}, \tbf{2)} \ttt{Greedy} scheme, which is equivalent to \epsg with $\epsilon=0$, \tbf{3)} \ttt{Uniform}, where an equal number of samples are drawn for each attribute, and \tbf{4)} \ttt{Uncurated}, where samples are drawn according to the natural distribution of the dataset. We note that, for some datasets, the natural distribution is uniform, for those we omit the results for the \ttt{Uncurated} scheme. 

We will make the code for these experiments available in the supplemental materials. All experiments were run for multiple trials--the number of which is indicated in the results--and we report the average over trials with shaded regions on plots indicating $\pm 1$ standard deviation. Experiments for image datasets were run on a single GPU provided by Google Colab or other shared computing resources.    



\textbf{Synthetic Dataset.} In this experiment, we compare how the $\pi_n(u)$ returned by the different algorithms converge to $\pi^*(u)$ for the two synthetic models introduced in Section~\ref{subsec:motivating_example} with  $\mc{F}$ chosen as the family of logistic regression (LR) classifiers. Since the feature space is two dimensional, we use the version of $\Aopt$, \epsg, and \empirical schemes in which we train (from scratch) the classifier in each round. For $\Aopt$, we use UCB given in~\ref{eq:UCB} with $c_0 = 0.1$ and for \epsg, we use $\epsilon=0.1$. These values of $c_0$ and $\epsilon$ are selected via a grid search. \\
Figure~\ref{fig:experiment0} shows how $\pi_n(u)$ changes with $n$ for the three algorithms averaged over $100$ trials. As expected, the algorithms with an exploration component (i.e.,~$\Aopt$ and \ttt{$\epsilon$-greedy}) eventually converge to the optimal $\pi^*(u)$ value in both cases, whereas the \ttt{Empirical} scheme that acts greedily often gets stuck with a wrong mixture distribution, resulting in high variability in its performance.



\tbf{ {\em Adult} Dataset. }
For the remaining experiments we find that, for properly tuned values of $\epsilon$ and $c_0$, both $\Aopt$ and \epsg attain comparable minimum test error.  So we omit the \epsg results for the purpose of clarity, see Appendix~\ref{appendix:epsilon_greedy} for a detailed comparison of the two algorithms. \\
We now analyze the performance of the remaining algorithms on a dataset from the UCI ML Repository~\citep{Dua19} that is commonly used in the fairness literature: the {\em Adult} dataset.  It is a low dimensional problem, so we use LR classifiers and the exact version of each sampling algorithm, where the optimal classifier is computed with each new sample.  \\
We set $\mc{Z}$ to be white men, non-white men, white women, and non-white women.  The minimum test accuracy, over all attributes, at each iteration is displayed in Figure~\ref{fig:LR_datasets}.  Each sampling scheme approaches $80.5\%$ accuracy as sample size grows, this matches the results achieved under the LR algorithms reported in Table 4a) in ~\citet{martinez2020minimax}.  This is approximately the maximum accuracy an LR classifier can achieve on the white male sub-population in isolation and so additional sampling, or other fairness algorithms, cannot improve on this performance in a minimax sense.  We do note, however, that the adaptive algorithms hold a sizable advantage over \texttt{Uniform} at small sample sizes. 


\begin{figure}[t]
     \centering
    \includegraphics[width=0.45\columnwidth]{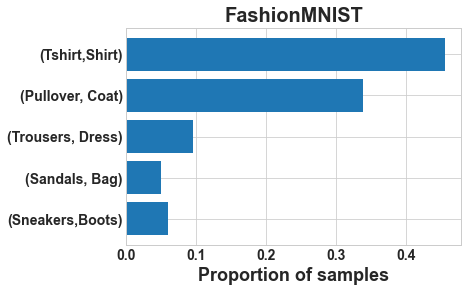}
    \caption{Mixture distribution learned by $\Aopt$ over 500 training rounds on {\em FashionMNIST}}
\label{fig:fmnistmix}
\end{figure}

\tbf{Image Datasets.} We also compare the performance of the different sampling schemes on larger scale problems and more complex hypothesis classes.  To this end we use three image datasets: {\em UTKFace}~\citep{zhifei17}, {\em FashionMNIST}~\citep{Xiao18}, and {\em Cifar10}~\citep{Krizhevsky09} with CNNs. 
We use the heuristic variants of the $\Aopt$, with $c_1=0$, and \ttt{Greedy} algorithms with batch sizes of 50. The CNN architecture, data transforms, and further algorithm parameters are detailed in Appendix~\ref{appendix:experiments}.

\begin{figure}[b]
\captionsetup[subfigure]{labelformat=empty}
     \centering
    \subcaptionbox{
    \label{fig:fashionmnist}}{\includegraphics[width=0.43\columnwidth]{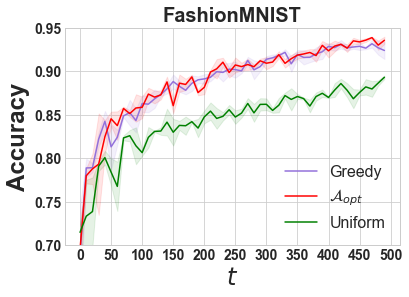}}
    \hfill   
    \subcaptionbox{
    \label{fig:face}}{\includegraphics[width=0.43\columnwidth]{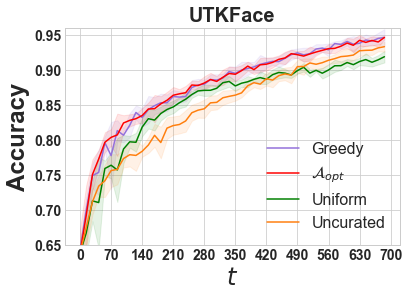}} \\ 
\vspace{-2.5em} 
\caption{Minimum test accuracy, over all attributes, for both {\em FashionMNIST} and {\em UTKFace} as a function of the time step $t$, averaged over 10 trials. }
\label{fig:real_datasets}
\vspace{-.5cm}
\end{figure}

The {\em UTKFace} dataset consists of face images annotated with age, gender, and ethnicity. We choose $Y = \{\ttt{Male,Female}\}$ and set $\mc{Z}$ to the five ethnicities. The minimum test accuracy, over all attributes, is shown in Figure~\ref{fig:face} and demonstrates a clear separation between the adaptive algorithms and both \ttt{Uniform} and \ttt{Uncurated} sampling.  \ttt{Uncurated} does particularly poorly in the low sample regime here, in contrast with the {\em Adult} dataset, where \ttt{Uncurated} performed comparably to the adaptive algorithms.  This is because, for {\em Adult} dataset, the lowest accuracy attribute is over-represented in the dataset with white men at $63\%$ of all samples.  Whereas, for {\em UTKFace}, the accuracy was lowest for Asian people, who are under-represented in the available data at only $14\%$ of all samples.

The other two datasets, {\em FashionMNIST} and {\em Cifar10}, were chosen to provide a controlled setting to demonstrate the existence of hard and easy attributes on real-world data. For both, we divide the labels into 5 pairs and assign each an "attribute".  The pairs were chosen according to existing confusion matrices to have pairs that are both easy and hard to distinguish from each other, see Appendix~\ref{appendix:experiments} for more details.  

For {\em FashionMNIST}, each pair in Figure~\ref{fig:fmnistmix} is one attribute and within each pair one item is assigned $Y=0$ and the other $Y=1$.  Then a single binary CNN classifier is learned simultaneously over all 5 pairs.  Figure~\ref{fig:fmnistmix} shows the mixture distribution generated by the $\Aopt$ sampling scheme, which allocated the vast majority of samples to the (Tshirt, Shirt) and (Pullover, Coat) pairs.  This makes intuitive sense, since both pairs of items are qualitatively very similar to each other, and aligns with common confusion matrices which indicate that those items are frequently misclassified as each other by standard classifiers.  Figure~\ref{fig:fashionmnist} displays the worst case accuracy for each  scheme on {\em FashionMNIST} as a function of time step and shows that both adaptive algorithms outperform \ttt{Uniform} sampling throughout the training process.  Finally, Figure~\ref{fig:fmnistattributes} shows the test accuracy on each attribute for both $\Aopt$ and \ttt{Uniform} sampling schemes.  $\Aopt$ maintains a much smaller spread between the accuracy over different attributes.  This equity is particularly desirable from a fairness perspective as it ensures no attribute will have a large advantage over any other, in the sense of expected accuracy.  It also illustrates our assumption that the optimal distribution will tend to equalize the losses across attributes. 

\begin{figure}[t]
\captionsetup[subfigure]{labelformat=empty}
     \centering
    \subcaptionbox{
    \label{fig:fmnistopt}}{\includegraphics[width=0.43\columnwidth]{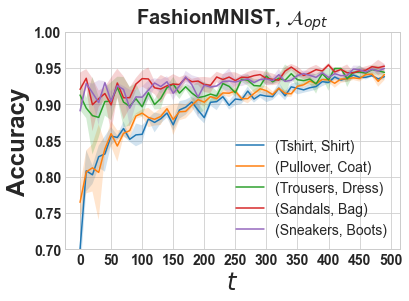}}
    \hfill   
    \subcaptionbox{
    \label{fig:fmnistunif}}{\includegraphics[width=0.43\columnwidth]{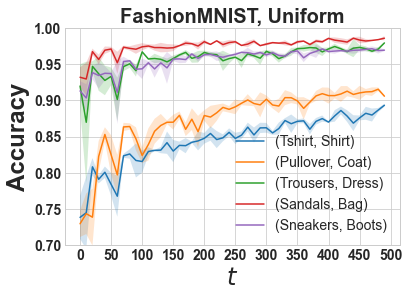}} \\ 
\vspace{-2.6em} 
\caption{Test accuracy for each attribute in {\em FashionMNIST} as a function of the time step, $t$, for both $\Aopt$ and \ttt{Uniform} sampling schemes, averaged over 10 trials. }
\label{fig:fmnistattributes}
\vspace{-.25cm}
\end{figure}

Final accuracy for all experiments, per-attribute accuracy for {\em UTKFace}, all results for {\em CIFAR10}, and details of the {\em Adult} dataset are included in Appendix~\ref{appendix:experiments}, along with results for the {\em German} dataset, another UCI ML Repository dataset.

         
         
         
         
         


\tbf{Empirical results when Assumption 2 is violated.} 
\begin{figure}[t]
     \centering
    \includegraphics[width=0.45\columnwidth]{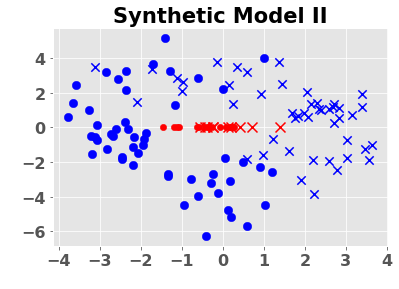}
    \caption{An instance of \toymodelii with attribute $Z=u$ shown in red and $Z=v$ shown in blue.}
    \label{fig:trenddata}
\end{figure}

\begin{figure}[b]
\vspace{-.5cm}
\captionsetup[subfigure]{labelformat=empty}
     \centering
    \subcaptionbox{
    \label{fig:trendacc0}}{\includegraphics[width=0.45\columnwidth]{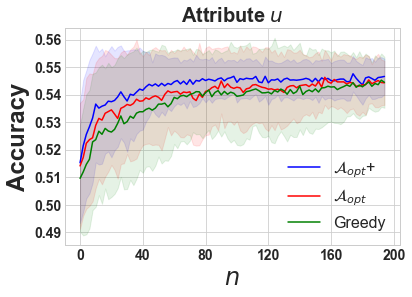}}
    \hfill   
\subcaptionbox{
\label{fig:trendacc1}}{\includegraphics[width=0.45\columnwidth]{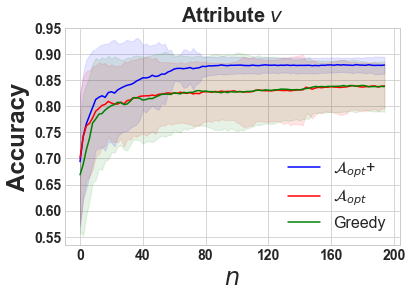}} \\ 
\vspace{-2.75em} 
\caption{Test accuracy as a function of sampling budget for both attributes from the dataset shown in Fig. \ref{fig:trenddata}, averaged over 100 trials.  
}
\label{fig:trendexperiment}
\vspace{-.5cm}
\end{figure}

Finally, we consider the case where our assumption of a unique, risk equalizing mixture distribution may not hold.  The practical effects of this situation can be seen in the {\em Adult} dataset where additional samples of the white male attribute show diminishing returns, while other attributes can attain higher accuracy given more samples.  In such a scenario, each adaptive algorithm will continue to select samples mainly from the worst case group despite this granting little improvement in performance.  

To evaluate this scenario we introduce a second synthetic data model: an instance of {\toymodelii} is illustrated in Figure~\ref{fig:trenddata} and specified in detail in Definition~\ref{def:toymodelii} in  Appendix~\ref{appendix:experiments}.  In this model we create one attribute, shown in red, which allows a relatively small maximal accuracy that can be attained with few samples.  In contrast, the blue attribute can be classified with high accuracy given few sample, as the bulk of its mass is in the central, separable clusters shown in the figure. But classification of this attribute also benefits from many additional samples as the sparser regions in the top left and bottom right are explored.  In this setting, constantly sampling from the attribute with the lowest empirical accuracy is inadvisable as additional samples of the red attribute cannot increase its accuracy, while additional samples of the blue attribute can increase its accuracy.  

We compared the performance of $\Aopt +$, the heuristic variation of our algorithm with a trend-based statistic included, $\Aopt$, and \texttt{Greedy} on {\toymodelii}.  Figure~\ref{fig:trendexperiment} shows the results, each of the three algorithms achieves similar minimax accuracy and performs worst on attribute $Z=u$, shown in red in Figure~\ref{fig:trenddata}.  But $\Aopt +$ achieves $4$ points higher average accuracy than both \texttt{Greedy} and $\Aopt$ with the original UCB on the other attribute, $Z=v$.  This demonstrates that the additional term in the $\Aopt +$ UCB is effective at recognizing when the hardest group is not benefiting from additional training samples and redistributing them more effectively.  


\section{Conclusion}
\label{sec:conclusion}
We considered the problem of actively constructing a training set in order to learn a classifier that achieves minimax fairness in terms of predictive loss. We proposed a new strategy for this problem~($\Aopt$) and obtained theoretical guarantees on its performance. 
We then showed that the theoretical performance achieved by $\Aopt$ and \epsg~\citep{abernethy2020adaptive} cannot be improved in general, by obtaining algorithm independent lower-bounds for the problem. 

Our experiments demonstrated that adaptive sampling schemes can achieve superior minimax risk [Fig.~\ref{fig:LR_datasets}, Fig.~\ref{fig:real_datasets}] and smaller disparity between per-attribute risks [Fig.~\ref{fig:fmnistattributes}] compared to both uniform and uncurated schemes.  The results in Fig.~\ref{fig:experiment0} show the necessity of exploration, as purely greedy algorithms can converge to sub-optimal mixture distributions.  And finally Fig.~\ref{fig:trendexperiment} shows the versatility of our general UCB-based strategy in its ability to readily incorporate new terms to accommodate the practical challenges posed by real-world problems.

\vspace{-0.5em}
Our theoretical results rely on the existence of a unique risk equalizing mixture distribution $\pi^*$, a condition that may not always hold. Thus, an important future work is to relax this assumption, and to design and analyze algorithms that identify pareto optimal mixture distributions that achieve minimax fairness. Another important direction is to derive and analyze active sampling strategies for other fairness measures. 


\newpage
\bibliographystyle{apalike}
\bibliography{ref}


\newpage


\newpage
\onecolumn 
\begin{appendix}
\newpage 

\section{Proof of Proposition~\ref{prop:sufficient}}
\label{proof:proposition_sufficient}
\begin{proof}
First we note that there must exist at least one optimum mixture distribution $\pi^*$ since by the \assumpref{assump:ideal_case}, the objective function is continuous and the domain of optimization, $\Delta_m$, is compact. We also note that any optimal $\pi^*$ must lie in the interior of the simplex $\Delta_m$ as a consequence of \assumpref{assump:ideal_case2}~(this follows by contradiction). 

Next, we show that any such $\pi^*$ must equalize the $L(z, f_{\pi^*} \coloneqq \mbb{E}_z \lb \loss \lp f_{\pi^*}, X, Y \rp \rb$ for all $z \in \mc{Z}$. Indeed, assume that this is not the case and enumerate the elements of $\mc{Z}$ as $z_1, \ldots, z_m$ in decreasing order of $L(z, f_{\pi^*})$ value. Also introduce $\varrho$ to denote  $\min_{z \neq z_1} L(z_1, f_{\pi^*}) - L(z, f_{\pi^*})$ and (for now) assume that $\varrho>0$. 

Now define $\pi_\epsilon$ as follows: $\pi_\epsilon(z) = \pi^*(z)$ for $z \in \mc{Z}' \coloneqq \mc{Z} \setminus \{z_1, z_m\}$, $\pi_\epsilon(z_1) = \pi^*(z_1) + \epsilon$ and $\pi_\epsilon(z_m) = \pi^*(z_m) - \epsilon$.  Due to the continuity assumption on the mapping $\pi \mapsto L(z,f_\pi)$ for all $z \in \mc{Z}$ in \assumpref{assump:ideal_case}, we note that there exists $\epsilon_0<\pi_{\min} \defined \min_{z \in \mc{Z}} \pi^*(z)$, such that for all $z\neq z_1$ we have $L(z, f_{\pi_{\epsilon}}) \leq L(z_1, f_{\pi^*}) - \varrho/2$ for all $\epsilon \leq \epsilon_0$. Note that $\pi_{\min}$ must be strictly greater than $0$, since $\pi^*$ lies in the interior of $\Delta_m$. Finally, due to the monotonicity assumption in \assumpref{assump:ideal_case}, there must exist $\varrho_1>0$ such that $L(z_1, f_{\pi_{\epsilon_0}}) = L(z_1, f_{\pi^*}) - \varrho_1 <L(z_1, f_{\pi^*})$. Together, these results imply that 
\begin{align*}
    \max_{z \in \mc{Z}} L(z, f_{\pi_{\epsilon_0}} ) &\leq \max \{ L(z_1, f_{\pi^*})- \varrho/2, \; L(z_1, f_{\pi^*}) - \varrho_1 \}  <  L(z_1, f_{\pi^*}),
\end{align*}
thus contradicting the optimality assumption on $\pi^*$. This completes the proof of the statement that any optimal $\pi^*$ must ensure that $L(z,f_{\pi^*}) = L(z', f_{\pi^*})$ for all $z, z' \in \mc{Z}$. 

Finally, the fact that $L(z, f_{\pi^*}) = L(z', f_{\pi^*})$ for any optimal $\pi^*$ also ensures the uniqueness of the optimum mixture distribution. Again, assume that this is not the case and there exists another optimal $\tilde{\pi}^* \neq \pi^*$ achieving optimum value in~\eqref{eq:optimal_classifier_mixture}. Then there must exist a $z$ such that $\pi(z) \neq \pi^*(z)$, and hence $L(z, f_{\pi^*}) \neq L(z, f_{\tilde{\pi}^*})$. Hence, by \assumpref{assump:ideal_case}, we have $\max_{z \in \mc{Z}} L(z, f_{\pi^*}) \neq \max_{z \in \mc{Z}} L(z, f_{\tilde{\pi}^*})$, which contradicts our hypothesis that both $\pi^*$ and $\tilde{\pi}^*$ achieve the optimum value in~\eqref{eq:optimal_mixture}. This concludes the proof of the uniqueness of $\pi^*$.
\end{proof}


\section{Details of \algoref{algo:algo1full}}

\subsection{Proof of Theorem~\ref{theorem:regret_simple}}
\label{proof:excess_risk1}
Recall the UCB first defined in Equation~\ref{eq:ucb}:
\begin{align}
    \label{eq:modified_ucb}
    U_t(z,\hat{f}_t) \coloneqq \frac{1}{|\mc{D}_z|} \sum_{(x,y) \in \mc{D}_z} \!\!\!\!\loss(\hat{f}_t, x, y) \;+\; e_{z}\lp N_{z,t}\rp
+  \frac{2C}{\pi_t(z)}\underbrace{\sum_{z' \in \mc{Z}} \pi_t(z') e_{z'}(N_{z',t})}_{\defined \rho_t}.
\end{align}
With the parameter $C$ as defined in Assumption~\ref{assump:ideal_case3} and sequences $e_z(N)$ satisfying Assumption~\ref{assump:ideal_case4}.  

Before proceeding further, we introduce the following notation: 
\begin{itemize}
    \item For any $z \in \mc{Z}$ and at any time $t \leq n$, we use $N_{z,t}$ to refer to the number of samples drawn corresponding to the feature $z$ prior to time $t$, i.e., $N_{z,t} = \sum_{i=1}^{t-1} \indi{z_i = z}$. 
    
    \item $\Ehatz{z, f} = \frac{1}{N_{z,t}} \sum_{(x,y) \in \Dz} \loss(f, x, y)$, denotes the empirical risk of a classifier $f$ on the samples corresponding to attribute $z$,
    
    \item $\Ez{f} = \mbb{E}_z \lb \loss(f, X, Y) \rb$ denotes the population risk of classifier $f$ corresponding to attribute $z$.
    
    \item $\rho_t = \sum_{z \in \mc{Z}} \pi_t(z) e_{z,N_{z,t}}$ will be used to quantify the uniform deviation of the classifier $\hat{f}_t$ from the corresponding optimal classifier $f_{\pi_t}$. Recall that $\pi_t$ is the mixture distribution of the $|\mc{Z}|$ attributes at time $t$ constructed by the algorithm, i.e., $\pi_t(z) = \frac{N_{z,t}}{t}$ for  $z \in \mc{Z}$.
\end{itemize}


Suppose at the end of $n$ rounds, the resulting mixture distribution is $\pi_n$. Then introduce the set of \emph{over-represented} attributes, $\Zover \subset \mc{Z}$, defined as $\Zover \coloneqq \{z \in \mc{Z} : \pi_n(z) > \pi^*(z)\}$ and refer to the set $\Zunder \coloneqq \mc{Z}\setminus \Zover$ as the \emph{under-represented} attributes. Note that $\Zover$ is empty only if $\pi_n = \pi^*$. In this case, the sampling algorithm  has learned the optimal mixing distribution resulting in zero excess risk. 
Hence, for the rest of the proof, we will assume that $\Zover$ is non-empty. Note that for any $z_0 \in \Zover$, \assumpref{assump:ideal_case} ensures that we must have $\Ez[z_0]{f_{\pi_n}}  < \Ez[z_0]{f_{\pi^*}}$ since $\pi_n(z_0) > \pi^*(z_0)$.

\begin{lemma}
\label{lemma:over_represented}
 Suppose $t_0 \leq n$ denotes the last time that \algoref{algo:algo1full} queried any feature  belonging to the subset $\Zover$, and denote the corresponding feature by $z_0$.  
Then, if  $n$ is large enough to ensure that $n^{1-\xi} \geq \lp \frac{1}{\pi_{\min}}\rp$ where $\pi_{\min} = \min_{z \in\mc{Z}} \pi^*(z)$  and that $2\max_{z \in \mc{Z}} e_{z}\lp (\pi_{\min} n)^\xi \rp \leq \epsilon_0$~(introduced in Assumption~\ref{assump:ideal_case3}),  we have the following with probability at least $1-\delta$: 
\begin{align}
    \Ez[z]{f_{\pi_{t_0}}} \leq \underbrace{\Ez{f_{\pi^*}}}_{\coloneqq M^*} \;  + \; \underbrace{2e_{z_0}( N_{z_0,t_0}) +  \lp \nicefrac{4C}{\pi^*(z_0)}\rp \rho_{t_0}}_{\coloneqq B_0}.  \label{eq:t0}
\end{align}
Recall that $\rho_{t} = \sum_{z} \pi_t(z) e_{z}(N_{z,t})$, $C$ is the constant from \assumpref{assump:ideal_case3}, and $\pi^*$ is the optimal mixture distribution defined in~\eqref{eq:optimal_mixture}.
\end{lemma}

\begin{proof}
Throughout this proof we assume that the $1-\delta/2$ probability event introduced while defining the UCB term in~\eqref{eq:ucb} occurs for the two sequence of samples drawn by \algoref{algo:algo1full}: the first used in updating $\mc{D}_t$ for training the classifier $\hat{f}_t$,  and the second used in updating $\lp \mc{D}_z \rp_{z \in \mc{Z}}$ for estimating the loss of $\hat{f}_t$. Thus, both of these events occur with probability at least $1-\delta$. 

First note  the following chain of inequalities: $\pi_{t_0}(z_0) \geq \pi_n(z_0)>\pi^*(z_0) \geq \pi_{\min}$. Then at time $t_0$, it must be the case that $N_{z_0, t_0} \geq \pi_{\min} n$ as $z_0$ belongs to $\Zover$. This, along with the fact that $n \geq (1/\pi_{\min})^{1/(1-\xi)}$ means that $N_{z_0, t_0}  \geq \pi_{\min} n \geq n^{-1+\xi} n = n^\xi$, and hence 
the query to  $z_0$ must have been made due to the condition in Line~\ref{algoline:choose} of \algoref{algo:algo1full}, and not due to the forced exploration step in Line~\ref{algoline:forced}.

Next, we have the following chain of inequalities for any $z \neq z_0$: 
\begin{align*}
\Ez[z]{f_{\pi_{t_0}}} & \stackrel{\text{(a)}}{\leq} \Ez[z]{\hat{f}_{t_0}} + \lvert \Ez[z]{\hat{f}_{t_0}} - \Ez[z]{f_{\pi_{t_0}}} \rvert \\ 
& \stackrel{\text{(b)}}{\leq} \Ez[z]{\hat{f}_{t_0}} + \lp \nicefrac{2C}{\pi_{t_0(z)}} \rp \rho_{t_0} \\
& \stackrel{\text{(c)}}{\leq} \Ehatz[t_0]{z, \hat{f}_{t_0}} + e_{z}(N_{z,t}) +  \lp \nicefrac{2C}{\pi_{t_0(z)}}\rp \rho_{t_0} \\
& \stackrel{\text{(d)}}{\leq} \Ehatz[t_0]{ z_0, \hat{f}_{t_0}} + e_{z_0}(N_{z_0,t}) +  \lp \nicefrac{2C}{\pi_{t_0}(z_0)}\rp \rho_{t_0} \\
& \stackrel{\text{(e)}}{\leq} \Ez[z_0]{\hat{f}_{t_0}} + 2e_{z_0}( N_{z_0,t_0}) +  \lp \nicefrac{2C}{\pi_{t_0}(z_0)}\rp \rho_{t_0} \\
& \stackrel{\text{(f)}}{\leq} \Ez[z_0]{f_{\pi_{t_0}}} + 2e_{z_0}( N_{z_0,t_0}) +  \lp \nicefrac{4C}{\pi_{t_0}(z_0)}\rp \rho_{t_0} \\
& \stackrel{\text{(g)}}{\leq} \Ez[z_0]{f_{\pi^*}} + 2e_{z_0}( N_{z_0,t_0}) +  \lp \nicefrac{4C}{\pi^*(z_0)}\rp \rho_{t_0} \\
& \stackrel{\text{(h)}}{=} \Ez{f_{\pi^*}} \;+ \; 2e_{z_0} (N_{z_0,t_0}) +  \lp \nicefrac{4C}{\pi^*(z_0)}\rp \rho_{t_0} \\
& = M^* + B_0.
%
\end{align*}
In the above display: \\
{\bf (a)} follows from an application of triangle inequality, \\
{\bf (b)} applies \assumpref{assump:ideal_case3} and uses the fact that $n$ is large enough to ensure that the suboptimality of $\hat{f}_{t_0}$ w.r.t. $f_{\pi_{t_0}}$ is no larger than $\epsilon_0$. \\
{\bf (c)} follows from the definition of event $\Omega_2$, \\
{\bf (d)} follows from the attribute selection rule in Line~\ref{algoline:choose}, \\
{\bf (e)} again uses the uniform deviation event $\Omega_2$, \\
{\bf (f)} follows from another application of \assumpref{assump:ideal_case3}, \\
{\bf (g)} uses the fact that $\pi_{t_0} (z_0) > \pi^*(z_0)$, the monotonicity condition from \assumpref{assump:ideal_case}, \\
{\bf (h)} uses the result of \propref{prop:sufficient} to write $\Ez[z_0]{f_{\pi^*}} = \Ez{f_{\pi^*}} \coloneqq M^*$.  
\end{proof}

Next, we show that for a large enough value of $n$, due to the forced exploration step~(Line~\ref{algoline:forced} in \algoref{algo:algo1full}), the values of $\pi_{t_0}(z)$ for $z \in \Zunder$ are not too small. 

\begin{lemma}
\label{lemma:forced_exploration}
For any given $A<\pi_{\min} = \min_{z \in \mc{Z}} \pi^*(z)$, there exists an $n_0 < \infty$~, defined in~\eqref{eq:n0} below, such that for all $n \geq n_0$, the following are true at time $t_0$~(recall that $t_0$ denotes the last time at which an attribute from $\Zover$ was queried by the algorithm):
\begin{align}
    &\pi_{t_0}(z) \geq \frac{(m-1)\pi^*(z)}{m} + \frac{A}{m}, \qquad \text{for all}\;\; z \in \Zunder,   \label{eq:pi_t0_bound}\\
 \text{and}\quad   &t_0 \geq n\; \lp \frac{\pi_{\min}}{2\pi_{\min} - A} \rp. \label{eq:t0_bound}
\end{align}
Recall that in the above display $m = |\mc{Z}|$.
\end{lemma}
\begin{proof}
First note that for any $z \in \Zunder$, we have 
\begin{equation*}
M^* \leq \Ez{f_{\pi_{t_0}}} \leq M^* + B_{0},
\end{equation*}
where we have used the notation $M^* = \Ez{f_{\pi^*}}$ introduced in~\eqref{eq:t0}. The left inequality above is due to the monontonicity condition of \assumpref{assump:ideal_case}, while the right inequality is from Lemma~\ref{lemma:over_represented}. 

Introducing the notation $\pi_{\min} = \min_{z \in \mc{Z}} \pi^*(z)$, we note that by definition $t_0 \geq \pi_{\min}n$. Recall that the term $B_0$ is defined as $B_0 = 2e_{z_0}(N_{z_0,t_0}) +  \lp \nicefrac{4C}{\pi^*(z_0)}\rp \rho_{t_0}$. Now, since $N_{z_0, t_0} \geq \pi_{\min} n$ and due to the monotonicity of $e_z(N_{z,t})$, we can upper bound $e_{z_0}( N_{z_0, t_0})$ with $e_{z_0}( \pi_{\min}n)$. Next, recall that $\rho_{t_0} =\sum_{z \in \mc{Z}} \pi_{t_0}(z) e_z(N_{z,t_0}) \leq \max_{z \in \mc{Z}} e_z(N_{z,t_0})$. Since $t_0 \geq N_{z_0,t_0} \geq \pi_{\min} n$, and due to the fact that forced-exploration step, i.e.,~Line~\ref{algoline:forced} of \algoref{algo:algo1full}, was not needed at time $t_0$, we must have $N_{z,t_0} \geq (\pi_{\min} n)^{\xi}$ for all $z \in \Zunder$. Thus the second term in the definition of $B_0$ can be simply bounded with $\nicefrac{4C}{\pi^*(z_0)} \max_{z \in \mc{Z}}e_z((\pi_{\min}n)^\xi) \leq \nicefrac{4C}{\pi_{\min}} \max_{z \in \mc{Z}}e_z((\pi_{\min}n)^\xi)$. Combining these two steps, we finally get that $B_0 \leq \lp \nicefrac{4C}{\pi_{\min}} + 2 \rp \max_{z \in \mc{Z}} e_z((\pi_{\min}n )^{\xi})$. 

By the monotonicity of the terms $e_{z}( N_{z, t})$ and $\rho_t$, we note that $\lim_{n \to \infty} B_0 = 0$, since $\lim_{n \to \infty} \max_{z \in \mc{Z}} e_z((\pi_{\min}n)^\xi) = 0$. 
Thus as $n$ goes to infinity, $\Ez{f_{\pi_{t_0}}}$ converges to the optimal value $M^*$, which by continuity of the mapping $\pi \mapsto \Ez{f_\pi}$ for all $z \in \mc{Z}$ implies that $\pi_{t_0} \to \pi^*$.  We can use this fact to define a sufficient number of samples, denoted by $n_0$, beyond which it can be ensured that $\pi_{t_0}$ satisfies the  statement in~\eqref{eq:pi_t0_bound}.
\begin{align}
\pi_{\min} &= \min_{z \in \mc{Z}}\; \pi^*(z); \qquad\qquad\qquad\; b = \inf \left \{\max_{z \in \mc{Z}}\;\Ez{f_\pi} - M^*: \|\pi^*-\pi\|_\infty > \frac{\pi_{\min}-A}{m} \right \}; \label{eq:a} \\
n_0 &\coloneqq \max \left\{n_0', \; \frac{1}{\pi_{\min}^2}\right\}; \qquad\quad n_0' \coloneqq \min \left \{n \geq 1 : \lp \frac{4C}{\pi_{\min}} + 2 \rp \max_{z\in\mc{Z}}\;e_z((\pi_{\min}n)^\xi)\; \leq\; b \right \}. \label{eq:n0}
\end{align}
Thus the definition of the term $b$, combined with the upper bound on $B_0$ due to Lemma~\ref{lemma:over_represented} and the forced-exploration rule ensure that for $n \geq n_0$, we must have $\pi_{t_0}(z) \geq \pi^*(z)/2$ for all $z \in \Zunder$ as required by~\eqref{eq:pi_t0_bound}. 

Since we will use the above computation of $n_0'$ several times, we formalize it in terms of the following definition. 
\begin{definition}[\smallestbudget]
\label{def:smallest_budget}
Given constants $c>0,\, p, q \text{ and } r \in (0,1]$,  the function \smallestbudget~returns the following
\begin{align*}
& \smallestbudget(c, p, q, r)  = \min \left \{ n \geq 1 :  \max_{z} e_z(N_{pq}) \, \leq \, \gamma/c \right \}, \quad \text{where }\\
& \gamma = \inf \left \{ \max_{z \in \mc{Z}} \Ez{f_\pi} - M^*: \|\pi^*-\pi\|_{\infty} > r \right\}  \quad \text{ and } \quad N_{pq} = (pn)^{q}.
\end{align*}
\end{definition}
Note that $n_0'$ in \eqref{eq:n0} is equal to $\smallestbudget\lp (4C/\pi_{\min}+2),\, \pi_{\min},\, \xi ,\, (\pi_{\min}-A)/m \rp$. 

For the proof of the statement in~\eqref{eq:t0_bound}, we note that since $\pi_{t_0}(z) \geq  \pi^*(z) - \frac{\pi_{\min}-A}{m}$ for all $z \neq z_0$, the time $t_0$ must satisfy the following: 
\begin{align*}
   t_0 &= N_{z_0, t_0} + \sum_{z \neq z_0} N_{z, t_0} \geq \pi^*(z_0) n + \sum_{z \neq z_0} \lp \pi^*(z) - \frac{\pi_{\min}-A}{m}\rp t_0  \\
&    = \pi^*(z_0) n +  \lp 1 - \pi^*(z_0) + \frac{(m-1)(\pi_{\min}-A)}{m}  \rp t_0. 
\end{align*}
This, in turn, implies that 
\begin{equation*}
t_0 \geq n \, \lp  \frac{ \pi^*(z_0)}{ \pi^*(z_0) + \frac{m-1}{m}(\pi_{\min}-A) } \rp  \geq n \, \lp \frac{\pi_{\min}}{2\pi_{\min} - A} \rp. 
\end{equation*}
This completes the proof of~\eqref{eq:t0_bound}. 
\end{proof}

We now state a basic result about the behavior of $\pi_t$:

\begin{lemma}
\label{lemma:increase}
Suppose the empirical mixture distribution is $\pi_{r}$ at some time $r$, and in the time interval $\{r+1, \ldots, t\}$ the agent only queries attributes $z$ from $\mc{Z}' \subset \mc{Z}$ such that $\sum_{z \in \mc{Z}'}\pi_r(z) < 1$. Then there exists at least one $z \in \mc{Z}'$ such that $\pi_t(z)>\pi_r(z)$. 
\end{lemma}

\begin{proof}
The statement follows  by contradiction. Assume that  the conclusion stated above is not true, and $\pi_t(z) \leq \pi_r(z)$ for all $z \in \mc{Z}'$.  Introducing $b_z$ to denote the number of times attribute $z \in \mc{Z}'$ is queried in the time interval $\{r+1,\ldots, t\}$, we then have
\begin{equation}
\label{eq:proof_claim1}
    \pi_t(z)  = \frac{ r \pi_r(z)  + b_z}{t} \; \Rightarrow \; (t-r) \pi_t(z) + r \underbrace{\lp \pi_t(z) - \pi_r(z) \rp}_{ \leq 0} = b_z \; 
     \Rightarrow \; (t-r) \pi_t(z) \geq b_z. 
\end{equation}
Note that  $\sum_{z \in \mc{Z}'} \pi_t(z) = \frac{t- r(1-\sum_{z}\pi_r(z))}{t} < 1$ by assumption that $\sum_{z \in \mc{Z}'} \pi_r(z) <1$. Combining this with~\eqref{eq:proof_claim1}, we get the required contradiction as follows:
\begin{equation*}
(t-r) > (t-r)\sum_{z \in \mc{Z}'}\; \pi_t(z)  \geq \sum_{z \in \mc{Z}'}\; b_z = (t-r). 
\end{equation*}
\end{proof}

    

Before proceeding, we introduce some notations. As stated earlier, due to the definition of the term $t_0$, we know that in the rounds $t \in \{t_0+1, \ldots, n\}$, the algorithm only queries the attributes belonging to the set $\Zunder$. If the set $\Zunder$ is empty, that means $t_0$ must be equal to $n$ and the algorithm stops there. Otherwise, the interval $\{t_0+1, \ldots, n\}$ can be partitioned into $\{t_0+1, \ldots, t_1\}$, $\{t_1+1, \ldots, t_2\}$, \ldots, $\{t_s +1, \ldots, n\}$ for appropriately defined $t_1, \ldots, t_s $ and  $s \leq |\mc{Z}|-1$ as follows. 
\begin{itemize}
\item First we introduce the term $\mc{Z}_t$ to denote the `active set' of attributes at time $t$, i.e., the set of attributes that are queried at least once after time $t$. Note that we have $\mc{Z}_{t_0} = \Zunder$. 
    \item   Then~(for $t \geq t_0$) we define a subset $\Zover^{(1)}$ of $\mc{Z}_t$ as those attributes $z \in \mc{Z}_t$ for which we have $\pi_{t_0}(z)<\pi_n(z)$. By Lemma~\ref{lemma:increase}, we know that $\Zover^{(1)}$ must be non-empty. 
    
    \item Next, we define $t_1$ as the last time $t \leq n$ at which an attribute $z \in \Zover^{(1)}$ is queried by the algorithm. 
    
    \item If $t_1=n$, then we stop and $s=1$. Otherwise, we repeat the previous two steps with $\mc{Z}_{t_1} = \mc{Z}_{t_0} \setminus \Zover^{(1)}$. 
\end{itemize}

\noindent To clarify the above introduced notation, we present an example. 

\begin{example}
\label{example:phases}
Consider a  problem with set of attributes $\mc{Z} = \{\att_1, \att_2, \att_3, \att_4\}$, $n=20$ and $\pi^* = \lp 0.25, 0.25, 0.25, 0.25 \rp$. 
Suppose an adaptive algorithm\footnote{note that the sequence of attributes have been chosen only for illustrating the notation, and do not satisfy the forced exploration condition in \algoref{algo:algo1full}} selects the following sequence of attributes (given a budget of $20$):
\begin{align}
    \label{eq:actions}
    \att_1,\;\att_2,\;\att_3,\;\att_4,\;\att_1,\;\att_1,\;\att_1,\;\att_1,\;\att_1,\;\att_1,\;\att_4,\;\att_1,\;\att_1,\quad \att_2,\;\att_2,\;\att_2,\;\att_3,\;\att_3,\;\att_3,\quad \att_4.
\end{align}
Then as we can see, the algorithm ends up with $\pi_n = (9/20, 1/5, 1/5, 3/20)$. 
\begin{itemize}\itemsep0em
    \item Comparing $\pi_n$ with $\pi^*$, we observe that the set $\Zover$ is $\{\att_1\}$ and $\Zunder = \{\att_2, \att_3, \att_4\}$. 
    
    \item The last time an element of $\Zover$ is queried, that is $t_0$, is equal to $13$ and the corresponding attribute is $z_0 = \att_1$. The mixture distribution at time $t_0$ is $\pi_{t_0} = (9/13, 1/13, 1/13, 2/13)$. 
    
    \item Comparing $\pi_{t_0}$ with $\pi_n$, we observe that the set $\Zover^{(1)}$ is $\{\att_2, \att_3\}$  since their fractions increase in the period $[t_0+1, n]_\N$, and $\Zunder^{(1)} = \{\att_4\}$. The last time an element of $\Zover^{(1)}$ is queried is $t_1 = 19$, and the corresponding attribute is $z_1 = \att_3$. 
    
    \item Finally, since only one element remains, we have $\Zover^{(2)} = \{\att_4\}$,  $\Zunder^{(2)} = \emptyset$ and thus $t_2 = n$ and $z_2 = \att_4$. Also note that the total number of phases above is $3$ and hence the term $s=3-1=2$.
\end{itemize}

    
    

\end{example}

\begin{lemma}
\label{lemma:one_step}
Suppose, $z_1$ is the element of $\Zover^{(1)}$~(introduced above) queried by the algorithm at time $t_1$. Then, the following is true at time $t_1$ for all $z \in \mc{Z} \setminus\{z_1\}$
\begin{align}
    \Ez{f_{\pi_{t_1}}} \leq M^* + B_0 + B_1, \qquad \text{where} \qquad B_1 = 2e_{z_1}(N_{z_1,t_1}) +  \lp \nicefrac{8C}{\pi^*(z_1)}\rp \rho_{t_1} \label{eq:one_step}
\end{align}
Furthermore, repeating this process till the budget is exhausted, we get for any $z \in \mc{Z}$ and an $s< m$
\begin{align}
    \Ez{f_{\pi_n}} \leq M^* + \sum_{i=0}^s B_i, \qquad \text{where} \qquad B_i = 2e_{z_i, N_{z_i,t_i}} +  \lp \nicefrac{8C}{\pi_{\min} \pi^*(z_i)}\rp \rho_{t_i}, \;\; \text{ for } \; i>1. \label{eq:multi_step}
\end{align}
Here $s$ is the (random) number of `phases'~(see Example~\ref{example:phases}), and is always upper bounded by $m =|\mc{Z}|$. 
\end{lemma}

\begin{proof}
To prove~\eqref{eq:one_step}, we note that at time $t_1$ for any $ z \neq z_1$, we have 
\begin{align*}
\Ez{f_{\pi_{t_1}}} &\;\stackrel{\text{(a)}}{\leq}\; \Ez{\hat{f}_{t_1}} + \frac{2C}{\pi_{t_1}(z)} \rho_{t_1}  \;\leq\; \Ehatz[t_1]{z, \hat{f}_{t_1}} + e_{z, N_{z,t_1}}+  \frac{2C}{\pi_{t_1}(z)} \rho_{t_1} \\
& \;\stackrel{\text{(b)}}{\leq}\;\Ehatz[ t_1]{z_1, \hat{f}_{t_1}} + e_{z_1, N_{z_1,t_1}}+  \frac{2C}{\pi_{t_1}(z_1)} \rho_{t_1} \\
& \;\stackrel{\text{(c)}}{\leq}\;\Ez[z_1]{f_{\pi_{t_1}}} + 2e_{z_1, N_{z_1,t_1}}+  \frac{4C}{\pi_{t_1}(z_1)} \rho_{t_1} \\
& \;\stackrel{\text{(d)}}{\leq}\;\Ez[z_1]{f_{\pi_{t_0}}} + 2e_{z_1, N_{z_1,t_1}}+  \frac{4C}{\pi_{t_1}(z_1)} \rho_{t_1} \\
& \;\stackrel{\text{(e)}}{\leq}\;\Ez[z_1]{f_{\pi_{t_0}}} + 2e_{z_1, N_{z_1,t_1}}+  \frac{4C}{\frac{(m-1)\pi^*(z_1)}{m} + \frac{A}{m}} \rho_{t_1} \\
& \;\stackrel{\text{(f)}}{=}\;\Ez[z_1]{f_{\pi_{t_0}}} + B_1 
\;\leq\; M^* + B_0 + B_1. 
\end{align*}
In the above display \\
{\bf (a)} uses \assumpref{assump:ideal_case3} and the event $\Omega_2$ introduced while defining the UCB in~\eqref{eq:ucb}, \\
{\bf (b)} follows from the point selection rule in Line~\ref{algoline:choose} of \algoref{algo:algo1full}, \\
{\bf (c)} uses \assumpref{assump:ideal_case3}, \\
{\bf (d)} uses that from the definition of $z_1$, we must have $\pi_{t_1}(z_1) \geq \pi_{t_0}(z_1)$, and due to monotonicity assumption~(\assumpref{assump:ideal_case}), $\Ez[z_1]{f_{\pi_{t_0}}} > \Ez[z_1]{f_{\pi_{t_1}}}$, and \\
{\bf (e)} uses the fact that $\pi_{t_1}(z_1) \geq \pi_{t_0}(z_1) \geq \pi^*(z_1)/2$ proved in~\eqref{eq:pi_t0_bound}, and \\
{\bf (f)} uses Lemma~\ref{lemma:over_represented} to bound $\Ez[z_1]{f_{\pi_{t_0}}}$ with $M^* + B_0$ to get~\eqref{eq:one_step}. 

Now, assume that the budget $n$ satisfies $n \geq n_1 \coloneqq \max \{n_0, n_1'\}$; where 
\begin{align}
n_1'&= \smallestbudget\lp c_0 + c_1 ,\; \pi_{\min}, \; \xi, \; \frac{2(\pi_{\min} - A)}{m} \rp, \quad \text{ where }  \\
c_0 &= \frac{4C}{A_0} +2; \quad  c_1 = \frac{4C}{A_1}+2; \quad \text{and} \quad A_i \coloneqq \frac{i}{m}A + \frac{(m-i)\pi_{\min}}{m}, \; \text{for } 0 \leq i \leq m. 
\end{align} 
Then by the definition of the \smallestbudget~function~(introduced in Definition~\ref{def:smallest_budget} in Appendix~\ref{proof:excess_risk1}), we must have $\pi_{t_1}(z) \geq \frac{ (m-2)\pi^*(z)}{m} + \frac{2A}{m}$ for all $z \in \mc{Z}$. 

The proof of~\eqref{eq:multi_step} essentially follows by repeating the above argument a further $s-1$ times. However, there is one minor difference. In proving~\eqref{eq:one_step}, specifically in Step {\bf (e)}, we used the fact that $\pi_{t_1}(z_1)  \geq A_1 \coloneqq \frac{A}{m} + \frac{(m-1)\pi_{\min}}{m}$. For $i \geq 2$, we can similarly use the fact that $\pi_{t_i}(z_i) \geq A_i \coloneqq \frac{iA + (m-1)\pi_{\min}}{m}$. The corresponding requirement on $n$ is that, with
\begin{align}
n &\geq n_i \coloneqq \max \{n_0, \ldots, n_{i-1}, n'_i\}, \qquad \text{ where } \label{eq:ni}\\ 
   n'_i  &= \smallestbudget\lp c_0 + \ldots c_i,\; \pi_{\min},\; \xi,\; \frac{ (i+1)(\pi_{\min}-A)}{m} \rp, \qquad \text{and} \nonumber \\
   c_i &= \frac{4C}{A_i} + 2, \qquad \text{ for all } \;\; 0 \leq i \leq m. \nonumber
\end{align}
%

\end{proof}

\begin{lemma}
\label{lemma:final_step}
Finally, we obtain that $\max_{z \in \mc{Z}} \Ez{\hat{f}_{\pi_n}} - M^* = \mc{O} \lp \frac{ |\mc{Z}| C}{\pi_{\min}} \max_{z \in \mc{Z}} e_z(N_A) \rp$, where $N = A \lp \frac{\pi_{\min}}{2\pi_{\min}- A} \rp n$. 
\end{lemma}

\begin{proof}
First note that for all $t_i$, $i=1,2,\ldots, s$, we have $\rho_t \leq \max_{z \in \mc{Z}} e_z(N_{z,t})$. Now, we note the following: for any $i = 0, \ldots s$, we have $\min_{z \in \mc{Z}} \pi_{t_i}(z) \geq A_i \geq A$. Also for all $i=0,\ldots s$, we also have trivially, using~\eqref{eq:t0_bound},  $t_i \geq N_A \coloneqq \pi_{\min}/(2\pi_{\min}-A) n$. Together these two statements imply the following: 
\begin{itemize}
    \item $B_0 \leq c_0 \max_{z \in \mc{Z}} e_z(N_{z,t_0}) \leq c_0 \max_{z \in \mc{Z}} e_z(N_A) =\lp \frac{4C}{\pi_{\min}} + 2 \rp \max_{z \in \mc{Z}} e_z(N_A)$.
    
    \item Similarly, for $i \geq 1$ we have  $B_i \leq c_i \max_{z \in \mc{Z}} e_z(N_{z,t_i}) \leq c_i \max_{z \in \mc{Z}} e_z(N_A) \leq \lp \frac{4C}{A_i} + 2 \rp \max_{z \in \mc{Z}} e_z(N_A)$. 
\end{itemize}
Combining the above two points, we get the following 
\begin{equation*}
   \sum_{i=1}^{s} B_i \leq \sum_{i=1}^{s} \lp \frac{4C}{A_{i-1}} + 2 \rp \max_{z \in \mc{Z}} e_z(N_A) \; 
\end{equation*}
which, if $A \geq \pi_{\min}/2$, implies 
\begin{equation*}
\max_{z \in \mc{Z}} \Ez{f_{\pi_n}} - M^* = \mc{O} \lp \frac{ |\mc{Z}| C}{\pi_{\min}} \max_{z \in \mc{Z}} e_z(N_A) \rp
\end{equation*}
as required. Note that the above result holds under the assumption that $n$ is large enough to ensure that the conditions in~\eqref{eq:ni} is satisfied for all $0 \leq i \leq s$. Since $s \leq m-1$, a sufficient condition for this is that the condition in~\eqref{eq:ni} is satisfied for all $0 \leq i \leq m-1$. 
\end{proof}
\subsubsection{Knowledge of parameter $C$}
\label{appendix:parameterC}
In our analysis above, we did not impose any condition on the forced exploration parameter $\xi$; instead we used knowledge of the parameter $C$ from Assumption~\ref{assump:ideal_case3}.  We now show that if  for some $0<\xi <1$, it is known that $\lim_{N \to \infty} \max_{z \in \mc{Z}} \frac{e_z(N^\xi)}{N^{\xi-1}} = 0$, then we can remove the $C$ dependent term from~\eqref{eq:ucb} and obtain the same guarantees for $\Aopt$ as in Theorem~\ref{theorem:regret_simple}.

The proof will follow the same outline as in the previous section, and to avoid repetition, we obtain a result analogous to Lemma~\ref{lemma:over_represented}.  
In particular,  with the same definition of $t_0$ and $z_0$ as in Lemma~\ref{lemma:over_represented} and Lemma~\ref{lemma:forced_exploration}, the following is true at time $t_0$:
\begin{align*}
    \Ez{f_{\pi_{t_0}}} & \leq \Ehatz[t_0]{z, f_{\pi_{t_0}}} + e_{z}(N_{z, t_0}) \\
    & \leq \underbrace{\Ehatz[t_0]{z, \hat{f}_{t_0}} + e_{z}(N_{z, t_0})}_{ = U_{t_0}(z, \hat{f}_{t_0})}  + \frac{2C}{\pi_{t_0}(z)}\rho_{t_0} \\
    & \stackrel{(a)}{\leq} U_{t_0}(z_0, \hat{f}_{t_0}) +  \frac{2C}{\pi_{t_0}(z)}\rho_{t_0}  \\
    & \leq \Ez[z_0]{f_{\pi_{t_0}}}+ 2 e_{z}(N_{z_0, t_0}) + 2C \rho_{t_0} \lp   \frac{1}{\pi_{t_0}(z_0)}+ \frac{1}{\pi_{t_0}(z)} \rp \\ 
    &\stackrel{(b)}{\leq}  \Ez[z_0]{f_{\pi_{t_0}}} +\lp \frac{2 C}{\pi_{\min}}  + \frac{2C}{t_0^{\xi-1}} \rp \rho_{t_0}\\ 
    & \stackrel{(c)}{\leq} \Ez[z_0]{f_{\pi_{t_0}}} +\underbrace{\lp \frac{2 C}{\pi_{\min}}  + \frac{2C}{t_0^{\xi-1}} \rp \max_{z \in \mc{Z}} e_z ( t_0^{\xi} ) }_{\coloneqq B_0'} \\ 
\end{align*}
In the above display, \\
\tbf{(a)} uses the fact that at time $t_0$, the attribute $z_0$ was selected by maximizing $U_{t_0}(z, \hat{f}_{t_0})$, \\
\tbf{(b)} uses the fact that $\pi_{t_0}(z_0) > \pi_{\min}$ and  $\pi_{t_0}(z) \geq t_0^{\xi-1}$ for all $z \neq z_0$ since the forced exploration step was not invoked at time $t_0$, and \\
\tbf{(c)} uses the fact that $\rho_{t_0} \leq \max_{z \in \mc{Z}} e_{z}(N_{z,t_0}) \leq \max_{z \in \mc{Z}} e_z(t_0^\xi)$ due to the monotonicity of $e_z(N)$ and the fact that $N_{z, t_0} \geq t_0^{\xi}$ for all $z \in \mc{Z}$ at time $t_0$. 

Now, to continue with the rest of the proof as in Appendix~\ref{proof:excess_risk1}, we need that the term $B_0'$ converges to zero as $n$~(and hence, $t_0$) goes to infinity. The first term of $B_0'$, i.e., $2C \max_{z} e_z(t_0^\xi)/\pi_{\min}$, converges to zero from the condition that the uniform confidence bound converges to zero as the number of samples, $t_0^\xi$, goes to infinity; while the second term, $\frac{2C}{t_0^{\xi-1}} \max_{z}e_z(t_0^\xi)$ converges to zero from the assumption on $\xi$ made at the beginning of this section. Using this fact, we can then proceed as in Lemmas~\ref{lemma:forced_exploration},~\ref{lemma:one_step},~and~\ref{lemma:final_step} to get the final result.

\subsection{Relation to Active Learning in Bandits}
\label{appendix:alb}

Our formulation of the minimax fair classification problem diverges from the Active learning in bandit~(ALB) problem due to the fact that drawing samples from one attribute can reduce the performance of another.  To make the discussion concrete, consider an ALB problem with two distributions $Q_1 \sim N(\mu_1, \sigma_1^2)$ and $Q_2 \sim N(\mu_2, \sigma_2^2)$ for $\mu_1, \mu_2 \in \reals$ and $\sigma_1^2, \sigma_2^2 \in (0,\infty)$. Let  $N_{i,t}$ denote the number of samples allocated by an agent to $Q_i$ for $i=1,2$ and $t\geq 1$. Then, we can construct high probability confidence sequences around $\hat{\mu}_{i,t}$ (the empirical counterparts to $\mu_i$ at time $t$) of non-increasing~(in $t$) lengths $e_{i,t}$ such that $|\hat{\mu}_{i,t}-\mu_i|\leq e_{i,t}$. If at time $t$, the agent decides to draw a sample from arm $1$, it would not change the quality of its estimate of $Q_2$ at time $t+1$ and beyond.  

On the other hand, let us  consider Instance~2 of \toymodel$(\Mu)$~(introduced in \defref{def:toymodel}) with $\mc{F}$ being the set of all  linear classifiers passing through origin. This is illustrated in \figref{fig:alb}, where the circles denote the one-square-deviation region around the mean values. The shaded circles correspond to the attribute $Z=v$, and the circles with dashed boundaries correspond to the label $Y=0$. For example, $P_{X|YZ}(\cdot|Y=0,Z=u)$ is denoted by the top left circle.
Suppose at some time $t_1$, the current classifier $\hat{f}_{t_1}$~(represented by the solid gray line passing through origin in \figref{fig:alb}) has low~(resp. high) accuracy for the protected attribute $Z=v$~(resp. $Z=u$).  To remedy this, the agent may draw more samples from the distribution of attribute $Z=v$ to skew the training dataset distribution towards $Z=v$. This would result in the updated classifier $\hat{f}_{t_2}$~(the dashed gray line in \figref{fig:alb}) at some time $t_2>t_1$ to achieve high accuracy for the attribute $Z=v$. But this increased accuracy for $Z=v$ comes at the cost of a reduction in the prediction accuracy for the attribute $Z=u$, as shown in \figref{fig:alb}. 

The above discussion highlights the key distinguishing feature of our problem from the prior work in ALB. Because of this distinction, the existing analyses of the ALB algorithms do not carry over directly to our case. As a result, to quantify the performance of \algoref{algo:algo1full}, we device a new `multi-phase' approach which is described in  detail in  Appendix.~\ref{proof:excess_risk1}.

\begin{figure}[hbt]
    \centering
    \includegraphics[scale=0.4]{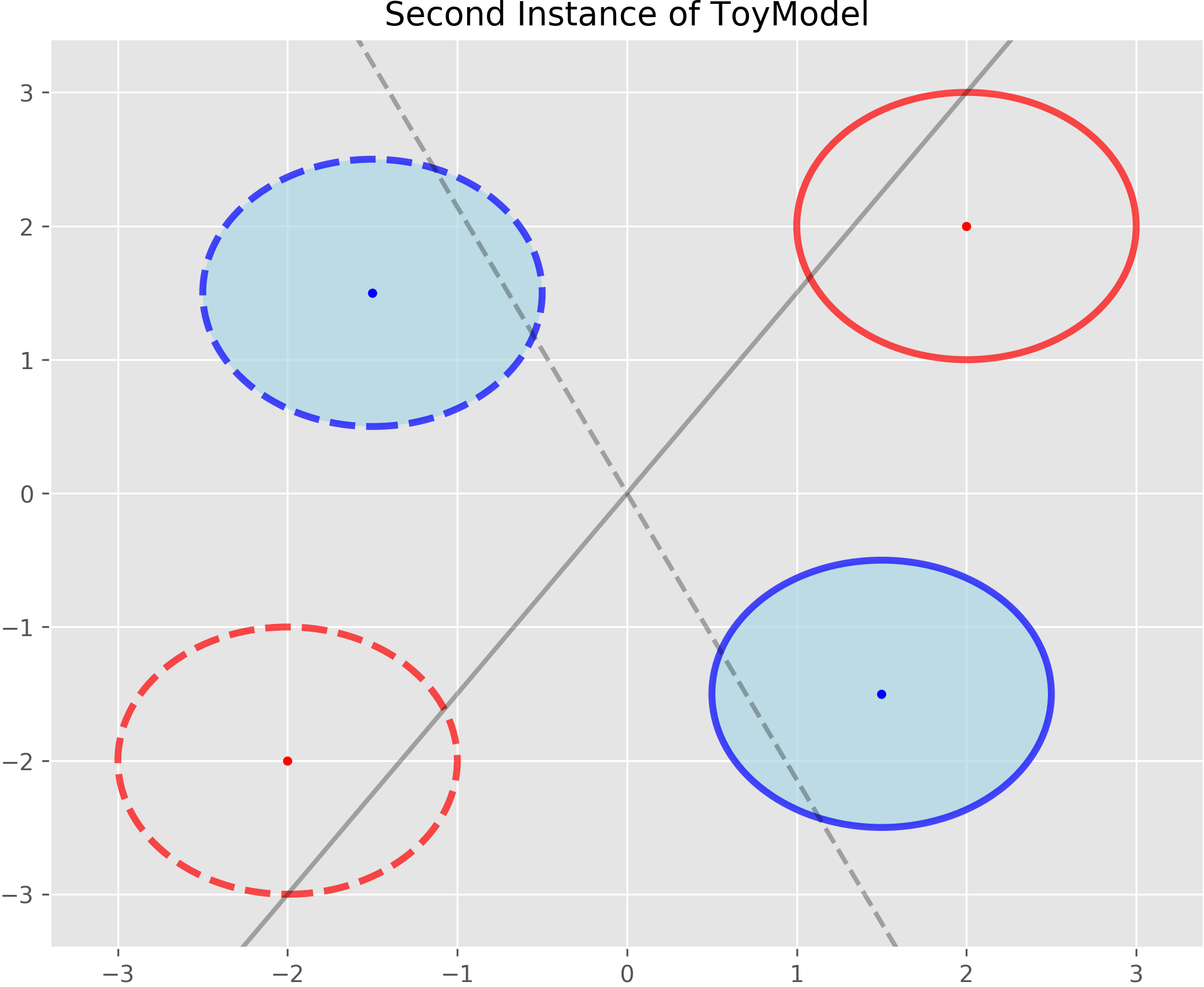}
    \caption{Figure demonstrates the second instance of \toymodel~introduced in Section~\ref{subsec:motivating_example}. Here the circles denote the one-standard-deviation regions of the distributions of features $X$ conditioned on $Y$ and $Z$. Shaded circles correspond to $Z=u$ and circles with dashed boundaries represent $Y=0$. The solid gray line represents a possible  linear classifier that may be learned if the dataset has fewer examples from $Z=v$. If an adaptive algorithm addresses this by populating the data-set with numerous examples from $Z=v$, a possible updated linear classifier is shown by the dashed gray line. This demonstrates the main distinguishing feature of our problem w.r.t. the active learning in bandits~(ALB) problem: \emph{allocating samples to improve the performance on one attribute can have an adverse effect on the performance of the resulting classifier on other attributes}.}
    \label{fig:alb}
\end{figure}

\newpage 
\section{Details of \texorpdfstring{\epsg}{epsilon-greedy} strategy}
\label{appendix:epsilon_greedy}

The \epsg strategy originally proposed by \citet{abernethy2020adaptive} proceeds as follows at time $t$: with probability $1-\epsilon$, draw a pair $(X_t, Y_t)$ from the distribution $P_{z_t}$ where $z_t$ is the attribute with largest empirical loss; and with probability $\epsilon$ draw $(Z_t, X_t, Y_t)$ from the population distribution $P_{XYZ}$. If $\pi_Z(\cdot)$ denotes the marginal of the population distribution over $\mc{Z}$, i.e., $\pi_{Z}(z) = \sum_{x, y} P_{XYZ}(x, y, z)$; then the \epsg strategy can be equivalently described as follows: \\
For any $t=1,2,\ldots$, do 
\begin{itemize}
    \item Draw a random variable $Q_t \sim \ttt{Bernoulli}(\epsilon)$. 
    \item If $Q_t = 1$, draw $Z_t \sim \pi_Z$. Else set $Z_t \in \argmax_{z \in \mc{Z}} \hat{L}_t(z, \hat{f}_t)$. 
    \item Draw a pair of samples from the distribution $P_{Z_t}$. 
\end{itemize}

For simplicity of presentation, we assume that $\pi_Z$ is the uniform distribution, i.e., $\pi_Z(z) = 1/m$ for all $z \in \mc{Z}$. However, our theoretical results can be easily generalized to the case of any $\pi_Z$ which places non-zero mass on all $z \in\mc{Z}$. 


\subsection{Analysis of the \texorpdfstring{\epsg}{epsilon-greedy} Sampling Strategy}
\label{subsec:analysis_epsilon_greedy}

The \epsg strategy proposed in~\citet{abernethy2020adaptive} proceeds as follows: for any $t \geq 1$, it maintains a candidate classifier $\hat{f}_t$ along with its empirical loss on $m$ validation sets corresponding to the attributes $z \in \mc{Z}$. At each round $t$, the algorithm selects an attribute $z_t$ as follows: with probability $\epsilon$, $z_t$ is drawn from a fixed distribution 
over $\mc{Z}$, and with probability $1-\epsilon$, it is set to the distribution with the largest empirical loss. Having chosen $z_t$, the algorithm draws a pair of samples from $P_{z_t}$ and updates the training and validation sets, as well as the classifier $\hat{f}_t$. This process is continued until the sampling budget is exhausted. 

\citet{abernethy2020adaptive} present two theoretical results on the performance of their \epsg strategy. The first one~\citep[Theorem~1]{abernethy2020adaptive}, is an asymptotic consistency result derived under somewhat restrictive assumptions: $\mc{X} = \mbb{R}$, $\mc{F}$ is the class of threshold classifiers, and the learner has access to an oracle which returns the exact loss, $L(z, f_{\pi_t})$, for every $\pi_t$. In their second result~\citep[Theorem~2]{abernethy2020adaptive}, they analyze the greedy version (i.e.,~$\epsilon=0$) of the algorithm with $|\mc{Z}| = 2$, and show that at time $n$, either the excess risk is of $\mc{O} \big(\max_{z \in \mc{Z}} \sqrt{ \nicefrac{ 2 d_{VC}\lp \loss \circ \mc{F} \rp \log (2/\delta)}{N_{z,n}}}\big)$, or the algorithm draws a sample from the attribute with the largest loss.
However, due to the greedy nature of the algorithm analyzed, there are no guarantees that $N_{z,n} = \Omega(n)$, and thus, in the worst case the above excess risk bound is $\mc{O}(1)$. 

We now show how the techniques we developed for the analysis of \algoref{algo:algo1full} can be suitably employed to study the \epsg strategy under the same assumptions as in Theorem~\ref{theorem:regret_simple}. In our results, we assume that the distribution according to which \epsg selects an attribute with probability $\epsilon$ at each round $t$ is uniform. This is just for simplicity and all our results hold for any distribution which places non-zero mass on all $z \in \mc{Z}$.

We first present an intuitive result that says if the \epsg strategy is too {\em exploratory} ($\epsilon$ is large), the excess risk will not converge to zero. We present an illustration of this result using Instance I of the $\toymodel$ introduced in Section~\ref{subsec:motivating_example} in Figure~\ref{fig:over_exploration} in Appendix~\ref{appendix:epsilon_greedy}. 


\begin{proposition}
\label{prop:epsilon_greedy}
If the \epsg strategy is implemented with $\epsilon> m \pi_{\min} := |\mc{Z}| \min_{z\in\mc{Z}} \pi^*(z)$, then its excess risk is $\Omega \lp 1 \rp$ with probability at least $1-\delta$ for $n$ large enough (see Equation~\ref{eq:tau0} in Appendix~\ref{appendix:over_exploration} for the precise condition on $n$).
\end{proposition}

\emph{Proof Outline.} 
The result follows from two observations: 
\tbf{(i)} For all $t$ larger than a term $\tau_0(\delta, \epsilon, \pi_{\min})$, defined in~\eqref{eq:tau0}, with  probability at least $1-\delta$, for any $z \in \mc{Z}$, we must have $\pi_t(z) \geq \frac{\pi_{\min} + \epsilon/m}{2}$, and
\tbf{(ii)} Suppose $\pi_{\min}$ is achieved by some attribute $z_{\min}$, i.e.,~$\pi^*(z_{\min}) = \pi_{\min}$. Then, the first statement implies that the excess risk of the \epsg algorithm is at least $\min_{\pi: \pi(z_{\min}) \geq (\pi_{\min}+\epsilon/m)/2}\; L(z_{\min}, f_{\pi}) - M^*$, which is strictly greater than zero by \assumpref{assump:ideal_case}. The detailed proof is reported in Appendix~\ref{appendix:over_exploration}. 
\hfill \qedsymbol

According to Proposition~\ref{prop:epsilon_greedy}, for the excess risk to converge to zero, the \epsg strategy must be implemented with $\epsilon \leq \pi_{\min}/m$. We next derive an upper-bound on the excess risk of \epsg, similar to the one in Theorem~\ref{theorem:regret_simple} for \algoref{algo:algo1full}. 

\begin{theorem}
\label{theorem:epsilon_greedy}
Let Assumptions~\ref{assump:ideal_case}-\ref{assump:ideal_case3} hold and \epsg implemented with $0<\epsilon<m\pi_{\min}$. If the query budget $n$ is sufficiently large (see Equations~\ref{eq:n0eps1} and~\ref{eq:n0eps2} in Appendix~\ref{appendix:over_exploration} for the exact requirements), then for any $0<\beta< \frac{m\pi_{\min}}{\epsilon}-1$, with probability at least $1-2\delta$, we have

\vspace{-0.15in}
\begin{small}
\begin{equation}
\label{eq:epsg_excess_risk}
\mc{R}_n(\text{\epsg})= \mc{O} \Big(\frac{|\mc{Z}|C}{q} \max_{z \in \mc{Z}} \; e_z(N_q) \Big), 
\end{equation}
\end{small}
\vspace{-0.175in}

where 
$q = \epsilon/m$, $N_q = qn\big(\pi_{\min} - q (1+\beta)\big)(1-\beta)$, and $C$ is the parameter introduced in Assumption~\ref{assump:ideal_case3}. 
\end{theorem}


\begin{remark}
\label{remark:epsg_bound}
The assumption on $\epsilon$ in Theorem~\ref{theorem:epsilon_greedy} implies that $q = \epsilon/m < \pi_{\min}$, and as a result $qn\pi_{\min} < \pi_{\min}^2n$. Given this and the monotonicity of the size of the confidence interval $e_{z}(N)$ w.r.t.~$N$, we may conclude that the bound on the excess risk of Alg.~\ref{algo:algo1full} (Equation~\ref{eq:excess-risk-Aopt}) is always tighter than the one for \epsg strategy (Equation~\ref{eq:epsg_excess_risk}). Note that for the classifiers with finite VC-dimension, the bound in~\eqref{eq:epsg_excess_risk} is of  the same order in $n$ as the one in~\eqref{eq:excess-risk-Aopt}, but with a larger leading constant. 
\end{remark}


\subsection{Proof of Proposition~\ref{prop:epsilon_greedy}}
\label{appendix:over_exploration}
\begin{figure}[ht]
    \centering
    \includegraphics[width=3in]{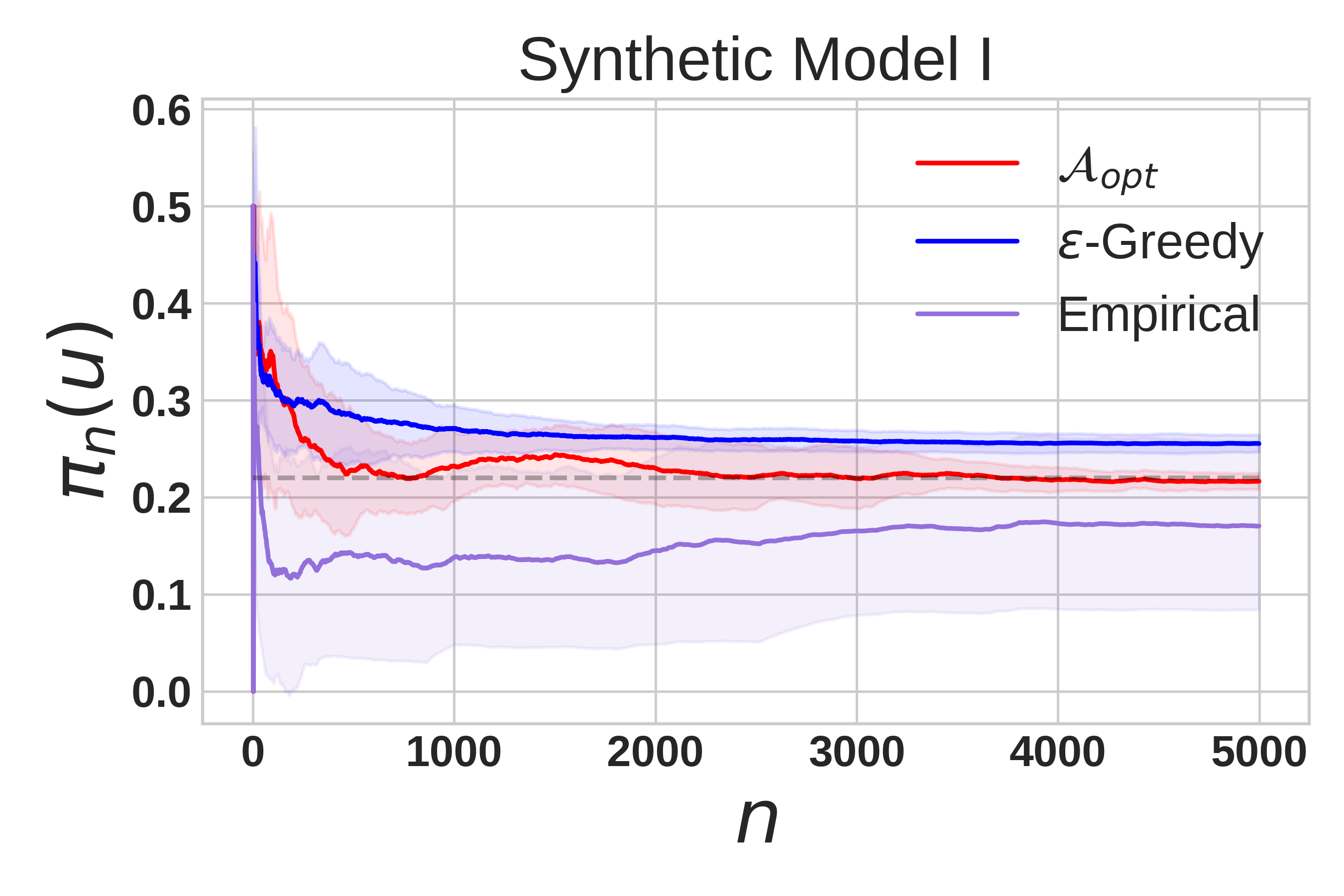}
    \caption{This figure shows an instance when the \epsg strategy is over-exploratory which results in the empirical mixture distribution $\pi_t$ being strictly away from $\pi^*$, and thus resulting in $\Omega(1)$ excess risk even as $n$ goes to infinity. For the \epsg strategy in this particular figure, we used $\epsilon=.5$ with $m=2$ and $\pi_{\min} = \min_{z } \pi^*(z) \approx 0.23$. This provides a numerical demonstration of the statement of Proposition~\ref{prop:epsilon_greedy}. Note that the $\Aopt$ strategy is more resilient to the wrong choice of the exploration parameter $c_0$: in this figure we used $c_0=1.0$~(much larger than the $0.1$ value used in experiments) and the corresponding $\pi_n(u)$ for $\Aopt$ still eventually converges towards $\pi^*(u)$.}
    \label{fig:over_exploration}
\end{figure}

Fix any $z \in \mc{Z}$, and decompose $N_{z,t}$ into $N_{z,t}^{(0)} + N_{z,t}^{(1)}$ where, $N_{z,t}^{(1)} = \sum_{i=1}^{t} \indi{z_i=z}Q_i$ is the number of times up to round $t$ the attribute $z$ was queried due to the exploration step~(i.e., $Q_t=1$) of the \epsg algorithm. 

Now, define $M_{z,0} = 0$ and $M_{z,t} = M_{z,t-1} + Q_t \indi{\tilde{z}_t=z} - \epsilon/m$. Due to the fact that $Q_t$ and $\tilde{z}_t$ are independent, it is easy to check that $(M_{z,t})_{t \geq 0}$ forms a martingale sequence. For some given $\delta>0$, introduce the notation $\delta_t = \frac{6\delta}{m \pi^2 t^2}$ and $b_t = \sqrt{ \frac{ t \log \lp \nicefrac{2}{\delta_t} \rp}{2} }$. We then have the following:
\begin{align*}
    \sum_{t=1}^n P\lp M_{z,t} > b_{t} \rp  & \leq  \sum_{t=1}^n \lp \prod_{i=1}^t \mbb{E}\lb \exp \lp \lambda Q_i \indi{\tilde{z}_i=z} - \epsilon/m \rp \middle \vert (Q_j)_{j=1}^{i-1} \rb \rp e^{-\lambda b_t}\\
    & \leq \sum_{t =1}^n e^{t \lambda^2/8 - \lambda b_t} \stackrel{(a)}{=} \sum_{t=1}^n e^{-2b_t^2/t} 
     \leq \sum_{t=1}^\infty \frac{\delta_t}{2}  = \frac{\delta}{2m}. 
\end{align*}
In the above display, \tbf{(a)} follows by setting $\lambda = 4b_t/t$. By repeating the same argument with the martingale sequence $\{-M_{z,t} : t \geq 0\}$, we get that $P\lp |M_{z,t}| \leq b_t, \forall t \geq 1, \; \forall z \in \mc{Z} \rp \geq 1- \delta$. In other words, it implies that, with probability at least $1-\delta$, we have the following: 
\begin{equation*}
    N_{z,t}^{(1)} \geq \frac{t}{\epsilon} - b_t, \quad \forall \; t \geq 1, \; \forall z \in \mc{Z}. 
\end{equation*}
This implies that $N_{z,t} \geq \nicefrac{\epsilon}{m} - b_t$ for all $z, t$. In particular, if $z_{\min}$ is the attribute such that $\pi^*(z_{\min}) = \min_{z \in \mc{Z}} \pi^*(z) = \pi_{\min}$, then this implies that $\hat{\pi}_t(z_{\min}) \geq \epsilon/m - b_t/t$ for all $t, z$. Using the fact that $\lim_{t \to \infty} b_t/t = 0$,  we  define $\tau_0 = \tau_0(\delta, \epsilon, \pi_{\min})$, as follows:
\begin{align}
\label{eq:tau0}
    \tau_0(\delta, \epsilon,\pi_{\min}) \coloneqq \min \left \{ t \geq 1 : \sqrt{ \frac{ \log\lp \nicefrac{m \pi^2 t^2}{3\delta} \rp}{2t}} \leq \frac{\epsilon - m\pi_{\min}}{2m} \right\}. 
\end{align}
This implies that with probability at least $1-\delta$, $\pi_t(z_{\min}) \geq (\pi_{\min}+\epsilon/m)/2$ for all $t \geq \tau_0$. Hence the excess risk of the \epsg algorithm under these conditions is at least $\min_{\pi: \pi(z_{\min}) \geq (\pi_{\min}+\epsilon/m)/2} \mbb{E}_{z_{\min}} \lb \loss \lp f_{\pi}, X, Y \rp \rb - M^*$, which is an $\Omega(1)$ term due to \assumpref{assump:ideal_case}.

\subsection{Proof of Theorem~\ref{theorem:epsilon_greedy}}
\label{appendix:proof_epsilon_greedy}
First we partition the interval $[1,n]_\N$ into $\Tau_0$ and $\Tau_1$, where $\Tau_j = \{ t: Q_t=j\}$ for $j=0,1$. In other words, $\Tau_0$ denotes the times at which the \epsg strategy is \emph{greedy} while $\Tau_1$ denotes the times at which $Q_t=1$ and the \epsg strategy is exploratory. 

Then we introduce the following terms: 
\begin{itemize}
    \item Define $\Zover$ and $\Zunder$ as in the proof of Theorem~\ref{theorem:regret_simple}. That is, $\Zover = \{z \in \mc{Z}: \pi_n(z) > \pi^*(z)\}$ and $\Zunder = \mc{Z} \setminus \Zover$. We assume, as before, that $\Zover \neq \emptyset$. Then,  we define $t_0 \coloneqq \max \{ t \in \Tau_0: z_t \in \Zover \}$, and denote the corresponding attribute (queried at time $t_0$) with $z_0$. 
    
    \item Then, by appealing to Lemma~\ref{lemma:increase}, we note that there must exist a nonempty $\Zover^{(1)} \subset \Zunder$ such that $\pi_n (z) > \pi_{t_0}(z)$ for all $z\in \Zover^{(1)}$. Using this define $t_1 = \max \{ t \in \Tau_0: z_t \in \Zover^{(1)}\}$ and use $z_1$ to denote the corresponding attribute queried at time $t_1$. We can proceed in this way to define $\{ (t_j, z_j): 2 \leq j \leq s\}$ for an appropriate $s \leq m-1$~(such that $t_s=n$). 
\end{itemize}

First, we define the $1-\delta$ probability event $\Omega_3$ as follows~(see proof of Proposition~\ref{prop:epsilon_greedy} in Appendix~\ref{appendix:over_exploration} for derivation): 
\begin{align}
\label{eq:omega3}
    \Omega_3 = \left \{ \left \lvert N_{z,t}^{(1)} - \frac{\epsilon}{m} t \right \rvert \leq \underbrace{\sqrt{ \frac{ t \log(m \pi^2 t^2/3\delta)}{2} }}_ {\coloneqq b_t} \right \}. 
\end{align}
For the rest of this proof, we will assume that the $(1-2\delta)$ probability event $\cap_{i=1}^3 \Omega_i$ holds, where $\Omega_1$ and $\Omega_2$ are the same uniform deviation results that were used in defining the UCB in~\eqref{eq:ucb}.  

Then we have the following: 

\begin{lemma}
\label{lemma:eps_greedy1}
 At time $t_0$, we have $\Ez{f_{\pi_{t_0}}} \leq M^* + \tilde{B}_0$ where $\tilde{B}_0  \coloneqq 4 \lp 1 + \nicefrac{C}{q - \beta_{t_0}/t_0} \rp \max_{z \in \mc{Z}} e_{z, N_{z, t_0}}$ and $q=\epsilon/m$.  
\end{lemma}

\begin{proof}
\begin{align*}
    \Ez{f_{\pi_{t_0}}} & \leq \Ehatz[t_0]{z, f_{\pi_{t_0}}} + 2e_{z}(N_{z, t_0}) \\
    & \leq \Ehatz[t_0]{z, \hat{f}_{t_0}} + 2e_{z}(N_{z, t_0})  + \frac{2C}{\pi_{t_0}(z)}\rho_{t_0} \\
    & \leq \Ehatz[ t_0]{ z_0, \hat{f}_{t_0}} + 2e_{z}(N_{z, t_0})  + \frac{2C}{\pi_{t_0}(z)}\rho_{t_0} \\
    & \leq \Ez[z_0]{f_{\pi_{t_0}}}+ 2\lp e_{z}(N_{z, t_0})  + e_{z_0}( N_{z_0, t_0}) \rp + 2C \rho_{t_0} \lp   \frac{1}{\pi_{t_0}(z_0)}+ \frac{1}{\pi_{t_0}(z)} \rp \\ 
    & \leq \Ez[z_0]{f_{\pi_{t_0}}}+ 4 \lp 1 + \frac{C}{\epsilon/m - b_{t_0}/t_0} \rp \max_{z \in \mc{Z}} e_{z}( N_{z, t_0})\\ 
    &  = \Ez[z_0]{f_{\pi_{t_0}}}+ \tilde{B}_0. 
\end{align*}
\end{proof}

\begin{lemma}
\label{lemma:eps_greedy2}
Suppose $n$ is large enough to satisfy the conditions in~\eqref{eq:n0eps2} for some fixed $0<\beta < \pi_{\min}/q - 1$. Then we have $\tilde{B}_0 \leq 4 \lp 1 + \nicefrac{C}{q(1-\beta)} \rp \max_{z \in \mc{Z}} e_{z}( N_q)$, where $N_q \coloneqq n q\lp \pi_{\min} - q(1+\beta) \rp  (1-\beta)$. 
\end{lemma}

Due to the definitions of event $\Omega_3$ and the time $t_0$, it must be the case that $N_{z_0, t_0} \geq N_{z_0, n}^{(0)} \geq \pi^*(z_0)n - N_{z_0, n}^{(1)} \geq \lp \pi^*(z_0) - \epsilon/m  = - b_n/n \rp n$. Next, we assume that $n$ is large enough, such that the following are satisfied for some $0<\beta$: 
\begin{itemize}
    \item 
First, we assume that $n$ is large enough to ensure that $b_n/n \leq \beta \epsilon/m$ for some $ 0 < \beta < m \pi_{\min}/\epsilon - 1$, i.e., 
\begin{align}
\label{eq:n0eps1}
    n \geq \tilde{n_0}(\beta) \coloneqq  \min \left \{ t \geq 1: \frac{\log t }{t} \leq \frac{ \beta^2 \epsilon^2}{m^2} - \frac{1}{2}\log\lp \nicefrac{m\pi^2}{3\delta} \rp \right \}.
\end{align}
This implies that $t_0 \geq n \lp \pi^*(z_0) - \epsilon/m(1+\beta) \rp \geq n \lp \pi_{\min} - \frac{\epsilon(1+\beta)}{m}\rp$. 

\item Next, we assume that $b_{t_0}/t_0 \leq \beta \epsilon/m$. A sufficient condition for this is 
\begin{align}
\label{eq:n0eps2}
n \lp \pi_{\min} - \frac{\epsilon(1+\beta)}{m} \rp \; \geq \; \tilde{n}_0(\beta) 
  \quad \Rightarrow \quad   n \; \geq \; \frac{ \tilde{n}_0(\beta)}{\pi_{\min} - \nicefrac{\epsilon(1+\beta)}{m}}.
\end{align}
\end{itemize}

Then, with the notation $N_q = n \lp \pi_{\min} - q (1+\beta) \rp \lp q(1-\beta)\rp$, where $q=\epsilon/m$, we have the following: 
\begin{equation*}
    \tilde{B}_0 = 4 \lp 1+ \frac{C}{q - b_{t_0}{t_0}} \rp \max_{z \in \mc{Z}} e_z(N_{z,t_0}) \leq 4 \lp 1+ \frac{C}{q(1-\beta)} \rp \max_{z \in \mc{Z}} e_z(N_q).
\end{equation*}
Now, proceeding in the same way as in Lemma~\ref{lemma:one_step} and Lemma~\ref{lemma:final_step}, we can define the terms $\tilde{B}_j$ for $j \geq 1$~(analogous to the terms $B_j$ introduced in Lemma~\ref{lemma:one_step}) to show that  with probability at least $1-2\delta$, the excess risk resulting from the \epsg strategy satisfies: 
\begin{equation*}
    \max_{z \in \mc{Z}} \Ez{f_{\pi_n}} - M^* \leq \sum_{j=0}^{s-1} \tilde{B}_j = \mc{O} \lp \frac{|\mc{Z}| C}{q (1-\beta)} \rp \max_{z \in \mc{Z}} e_z(N_q), 
\end{equation*}
as required.

\subsection{Comparison with \texorpdfstring{$\Aopt$}{A\_opt}}
The \epsg strategy differs from $\Aopt$ in two major ways: 
\begin{enumerate}
    \item The excess risk bound derived in Theorem~\ref{theorem:regret_simple} for $\Aopt$ is always tighter than the corresponding bound for \epsg strategy in Theorem~\ref{theorem:epsilon_greedy}. In particular, for the family of classifiers with finite VC dimension, both the algorithms achieve same convergence rates w.r.t. $n$, but the \epsg strategy has a larger leading constant. 
    
    \item From a practical point of view, the \epsg strategy is less robust to the choice of parameter $\epsilon$ as compared to the $\Aopt$ strategy. For instance, as shown in Figure~\ref{fig:over_exploration}, choosing a large value of $\epsilon$ may result in the mixture distribution~$(\pi_t)$ not converging to $\pi^*$, whereas even with much larger values of $c_0$, the $\pi_t$ from $\Aopt$ algorithm still eventually converges to $\pi^*$. 
\end{enumerate}



\section{Proof of Lower Bound }
\label{appendix:proof_lower} 

In this section, we first formally state the lower bound result and then present its proof.  We denote by $\mc{M}$ the class of $\toymodel$s introduced in Definition~\ref{def:toymodel}. We also define the following class of problems:

\begin{definition}
\label{def:Q}
Let $\mc{Q}$  denote the class of problems defined by the triplets $(\bm{\mu}, \mc{F}, \loss_{01})$, where $\bm{\mu} \in \mc{M}$ is an instance of the \toymodel, $\mc{F}$ is the class of linear classifiers in two dimensions, and $\loss_{01}$ is the $0-1$ loss. 
\end{definition}

For the function class $\mc{F}$, we know that $e_{z}(N) = \mc{O} ( \sqrt{ \log (n/\delta) / N} )$ for both $z \in \mc{Z} = \{u, v\}$, which implies that the expected excess risk achieved by both $\Aopt$ and \epsg strategies is of $\mc{O}\lp \sqrt{\log (n)/n} \rp$. We now prove  that this convergence rate (in terms of $n$) for this class of problems. 

\begin{proposition}
\label{prop:lower} 
Suppose $\mc{A}$ is any adaptive sampling scheme which is applied with a budget $n$ to a problem $Q \in \mc{Q}$ introduced in Definition~\ref{def:Q}. Then, we have 

\vspace{-0.2in}
\begin{small}
\begin{align}
    \max_{Q \in \mc{Q}}\; \mbb{E}_Q\lb \mc{R}_n \lp \mc{A} \rp  \rb = \Omega \lp 1/\sqrt{n} \rp. 
\end{align}
\end{small}
\vspace{-0.2in}
\end{proposition}

To prove this proposition, we consider two problem instances in the class of problems, $\mc{Q}$, used in the statement of Proposition~\ref{prop:lower}, denoted by $Q_\mu$ and $Q_\gamma$. The instance $Q_\mu$ has the synthetic model with mean vectors $\mu_{0u} = (-r, r), \mu_{1u} = (r, -r), \mu_{0v} = (-r', -r')$ and $\mu_{1v} = (r', r')$ for some $r'>r>0$, and $Q_\gamma$ has the mean vectors $\mu_{0u} = (-r', r'), \mu_{1u} = (r', -r'), \mu_{0v} = (-r, -r)$ and $\mu_{1v} = (r, r)$. For both these problem instances, implementing an adaptive sampling algorithm $\mc{A}$ with a budget $n$ induces a probability measure on the space $(\mc{X}\times\mc{Y} \times \mc{Z})^n$. We use $\mbb{P}_\mu$ and $\mbb{P}_\gamma$~(resp. $\mbb{E}_\mu$ and $\mbb{E}_\gamma$) to denote the  probability measures~(resp. expectations)  for the problem instances $Q_\mu$ and $Q_\gamma$ respectively. 

Then we have the following KL-divergence decomposition result. 

\begin{lemma}
\label{lemma:lower_proof1}
For any event $E$, we have the following:
\begin{equation*}
    \frac{n (r'-r)^2}{2} = D_{KL} \lp \mbb{P}_\mu, \mbb{P}_\gamma \rp \geq d_{KL} \lp \mbb{P}_\mu(E), \, \mbb{P}_\gamma(E) \rp \geq 2 \lp \mbb{P}_\mu(E) - \mbb{P}_\gamma(E) \rp^2. 
\end{equation*}
In the above display, $d_{KL}(a_1, a_2)$ for $a_1, a_2 \in [0,1]$ denotes the KL-divergence between two Bernoulli random variables with parameters $a_1$ and $a_2$ respectively. 

As a consequence of the above result, we have $|\mbb{P}_\mu(E) - \mbb{P}_\gamma(E)| \leq (r'-r)\sqrt{n/2}$. 
\end{lemma}

\begin{proof}
The first inequality is a  consequence of the data-processing inequality for KL-divergence \citep[Corollary~2.2]{polyanskiy2014lecture}, while the second inequality follows from an application Pinsker's inequality~\citep[Theorem~6.5]{polyanskiy2014lecture}. 

We now show the derivation of the first equality in the statement. 
Let $\mc{H}_t$ denote the history at the beginning of round~$t$, i.e., $\mc{H}_t = \lp Z_1, X_1, Y_1, \ldots, Z_{t-1}, X_{t-1}, Y_{t-1} \rp$. Then for any sequence of $\lp Z_1, X_1, Y_1, \ldots, Z_n, X_n, Y_n \rp$, we have 
\begin{align*}
    \mbb{P}_\mu \lp Z_1, \ldots, Y_n \rp = \prod_{t=1}^n \mc{A}\lp Z_t | \mc{H}_{t-1} \rp Q_{\mu}\lp Y_t|Z_t \rp Q_{\mu}\lp X_t | Y_t, Z_t \rp. 
\end{align*}
We can write a similar expression for $\mbb{P}_{\gamma}$ as well. Now, proceeding as in \citep[Lemma~15.1]{lattimore2020bandit}, we  get the following divergence decomposition result 
\begin{align*}
    D_{KL}\lp \mbb{P}_{\mu}, \mbb{P}_{\gamma} \rp &= \mbb{E}_{\mu} \lb N_{u, n} \rb D_{KL}\lp Q_{\mu}(\cdot, \cdot|Z=u),  Q_{\gamma}(\cdot, \cdot|Z=u) \rp  + \mbb{E}_{\mu} \lb N_{v, n} \rb D_{KL}\lp Q_{\mu}(\cdot, \cdot|Z=v),  Q_{\gamma}(\cdot, \cdot|Z=v) \rp  \\
    & = \lp  \mbb{E}_{\mu} \lb N_{u, n} \rb + \mbb{E}_{\mu} \lb N_{v, n} \rb \rp \frac{ (r'-r)^2}{2} = \frac{n (r'-r)^2}{2}, 
\end{align*}
where the last equality uses the expression for KL-divergence between two multi-variate Gaussian distributions.

\end{proof}

\begin{lemma}
\label{lemma:lower_proof2}
Suppose $\pi_n$ denotes the mixture distribution returned by the algorithm $\mc{A}$ after $n$ rounds when applied to a problem $Q \in \{ Q_\mu, Q_\gamma\}$. Introduce the event $E = \{ \pi_n(u) > 0.5 \}$ and assume that $r'<2r$. Then we have the following: 
\begin{equation*}
    \mbb{E}_\mu \lb \mc{R}_n (\mc{A}) \rb  \geq \frac{c_r}{2\sqrt{2}}(r'-r) \mbb{P}_\mu(E^c)  \quad\text{and} \quad 
    \mbb{E}_\gamma \lb \mc{R}_n (\mc{A}) \rb  \geq \frac{c_r}{2\sqrt{2}}(r'-r) \mbb{P}_\gamma(E), 
\end{equation*}
where $c_r \coloneqq \min_{x \in [r,2r]} \left \vert \Phi'\lp \frac{x}{\sqrt{2}} \rp \right \vert$ and $\Phi(\cdot)$ denotes the cumulative distribution function~(cdf) of a standard Normal random variable. 
\end{lemma}

\begin{proof}
We present the details only for the first inequality as the second inequality follows in an entirely analogous manner by replacing $E^c$ with $E$. 
\begin{align*}
    \mbb{E}_\mu \lb \mc{R}_n(\mc{A}) \rb &=     \mbb{E}_\mu \lb \mc{R}_n(\mc{A}) \indi{E} \rb +     \mbb{E}_\mu \lb \mc{R}_n(\mc{A}) \indi{E^c} \rb \\
    & \geq     \mbb{E}_\mu \lb \mc{R}_n(\mc{A}) \indi{E^c} \rb \\
    & \stackrel{(a)}{\geq} \lp \frac{ \Phi(r/\sqrt{2}) - \Phi(r'/\sqrt{2})}{2} \rp \mbb{P}_\mu (E^c) \\
    & \geq \min_{x \in [r,2r]}\left \vert \Phi(x/\sqrt{2}) \right \vert \lp \frac{ r' - r}{2 \sqrt{2}} \rp \mbb{P}_\mu\lp E^c \rp \\
    & = \frac{c_r}{2 \sqrt{2}}(r'-r) \mbb{P}_\mu\lp E^c \rp. 
\end{align*}
The key observation in the proof which relies on the choice of $\mc{F}$ as the set of all linear classifiers, is in step~\tbf{(a)} above. This step uses the fact that under the event $E^c$, when $\pi_n(u) \leq 0.5$, the minimax loss must be at least greater than $\frac{\Phi(r/\sqrt{2}) - \Phi(r'/\sqrt{2})}{2}$. Here $\Phi(\cdot)$ denotes the cdf of the standard normal random variable. 
\end{proof}

The final result now follows by combining the results of Lemma~\ref{lemma:lower_proof1} and Lemma~\ref{lemma:lower_proof2}. In particular, we have the following: 
\begin{align*}
\max_{Q \in \mc{Q}} \mbb{E}_Q \lb \mc{R}_n \lp \mc{A} \rp \rb \geq \max_{Q \in \{Q_\mu, Q_\gamma\} } \mbb{E}_{Q}[\mc{R}_n\lp \mc{A} \rp] & \stackrel{\text{(a)}}{\geq} \frac{1}{2} \lp \mbb{E}_\mu \lb \mc{R}_n \lp \mc{A} \rp \rb  + \mbb{E}_\gamma \lb \mc{R}_n \lp \mc{A} \rp \rb  \rp \\
& \stackrel{\text{(b)}}{\geq} \frac{c_r}{4\sqrt{2}} (r' - r) \lp \mbb{P}_\mu \lp E^c \rp + \mbb{P}_\gamma \lp E \rp \rp \\
& \stackrel{\text{(c)}}{\geq} \frac{c_r}{4\sqrt{2}} (r' - r) \lp 1-  \left \lvert \mbb{P}_\mu \lp E \rp - \mbb{P}_\gamma \lp E \rp  \right \rvert \rp \\
& \stackrel{\text{(d)}}{\geq} \frac{c_r}{4\sqrt{2}} (r' - r) \lp 1-  (r'-r)\sqrt{ \frac{n}{2}}\, \rp \\
& \stackrel{\text{(e)}}{\geq} \frac{c_r}{16 \sqrt{n}}.
\end{align*}
In the above display: \\
{\bf (a)} uses the fact that average is smaller than maximum, \\
{\bf (b)} lower bounds $\mbb{E}_\mu \lb \mc{R}_n \lp \mc{A} \rp \rb$ and $\mbb{E}_\gamma \lb \mc{R}_n \lp \mc{A} \rp \rb$ using the result of Lemma~\ref{lemma:lower_proof2}, \\
{\bf (c)} uses the fact that $\mbb{P}_{\mu}(E^c) + \mbb{P}_\gamma(E) =1 - \mbb{P}_{\mu}(E^c) + \mbb{P}_\gamma(E) \geq 1- |\mbb{P}_{\mu}(E^c) - \mbb{P}_\gamma(E)|$, \\
{\bf (d)} follows by using the bound $|\mbb{P}_{\mu}(E^c) - \mbb{P}_\gamma(E)| \leq (r'-r)\sqrt{n/2}$ derived in Lemma~\ref{lemma:lower_proof1}, and finally, \\ 
{\bf (e)} follows by setting $r' = r+ 1/\sqrt{2n}$. 
\section{Details of Experiments}
\label{appendix:experiments}

\subsection{Synthetic Datasets}

\emph{Data.} We used the two synthetic models introduced in Section~\ref{subsec:motivating_example} for generating the training set. Here we provide a formal definition for \toymodelii, an instance of which is illustrated in Figure~\ref{fig:trenddata}.

\begin{definition}[\toymodelii]
\label{def:toymodelii}
Set $\X = \mbb{R}^2$, $\mc{Y} = \{0, 1 \}$, and $\mc{Z} = \{u, v\}$ and fix an angle $\theta$ and a constant $a\in\mbb{R}$. Then $P_{X|Y=y,Z=u}=\mathcal{U}(c_y,c_y+a)$. That is, within attribute $Z=u$ and for either label $Y=y$, $X$ is uniformly distributed along one dimension.  To get the conditional distributions for $Z=v$ we draw from mixtures of three Gaussians: {\small $\widetilde{P}_{X=x|Y=y,Z=v} = \sum_{i=1}^3 \rho_{i,y} P_{i,y}(x)$} where each $P_{i,y} = \mathcal{N}(\mu_{i,y},\Sigma_{i,y})$ and $\rho_{i,y}$ are non-negative and sum to one over $i$.  The resulting sample is then uniformly rotated counterclockwise by $\theta$ in the plane.  
\end{definition}

To elaborate on the intuition provided for this model in the main text, our goal here is to design a problem where one attribute, here $Z=u$, has a low maximum expected accuracy, for a linear classifier, that can be attained with relatively few samples.  Here, any linear classifier will misclassify at least half of all points in the overlap of the intervals $(c_0,c_0+a)$ and $(c_1,c_1+a)$. And any linear classifier which intercepts this overlap will attain the maximal possible expected accuracy on $Z=u$. Then for $Z=v$ we choose the $\rho_{i,y}$ and covariances such that most mass is contained in easily separable clusters, so that a classifier can attain high accuracy on $Z=v$ with few samples.  But, the sparser clusters are chosen to have some overlap between $Y=1$ and $Y=0$.  The overlap and sparseness require a learner draw more samples from $Z=v$ to accurately learn the optimal boundary.  

Thus $\Aopt$ and \texttt{Greedy} will continually sample from $Z=u$, as any linear classifier will have relatively low accuracy on it, despite the fact that additional samples do not improve performance on this attribute.  In contrast, $\Aopt +$ identifies this stagnation in performance and samples from $Z=v$, over time drawing sufficient samples from the sparse clusters to maximize accuracy on both attributes.   

\emph{Classifier.} We set $\mc{F}$ to the set of logistic regression classifiers. More specifically we used the \ttt{LogisticRegression} implementation in the \ttt{scikit-learn} package. 

\emph{Details of experiments.} For the experiments on \toymodel, pictured in Figure~\ref{fig:experiment0}, for each trial we set $n=1000$ and tracked the $\pi_t(u)$ values for the three algorithms. We repeated the experiment for $100$ trials and plot the resulting mean $\pi_t(u)$ values in the curves, along with the one-standard-deviation regions. 

For experiments on \toymodelii, pictured in Figure~\ref{fig:trendacc0} and Figure~\ref{fig:trendacc1}, we used the heuristic algorithm, $\Aopt+$, with $c_1=0.1$ and the Mann-Kendall statistic tracking the accuracy after the previous 20 sample draws for each attribute.  We ran 100 trials for each scheme and report the mean accuracy over both attributes, along with one-standard-deviation regions.

\subsection{Real Datasets} 

We used the following datasets:
\begin{itemize}
    \item \href{https://susanqq.github.io/UTKFace/}{\emph{UTKFace}} dataset. This dataset consists of large number of face images with annotations of age, gender and ethnicity. We used $Y = \{ \ttt{Male, Female}\}$ and set $\mc{Z} = \{ \ttt{White, Black, Asian, Indian, Other}\}$. 

    \item \href{https://github.com/zalandoresearch/fashion-mnist}{\emph{FashionMNIST}} dataset. This dataset  consists of $10$ different classes, which were  paired off to get five different binary classification tasks: \{\ttt{(Tshirt, Shirt), (Trousers, Dress), (Pullover, Coat), (Sandals, Bag), (Sneakers, AnkleBoots)}\}
    
    \item \href{https://www.cs.toronto.edu/~kriz/cifar.html}{\emph{Cifar10}} dataset. This dataset also consists of $10$ different classes, which were paired off as follows: \{ \ttt{(airplane, ship), (automobile, truck), (bird, cat), (deer, dog), (frog, horse)}\}.
    
    \item \href{https://archive.ics.uci.edu/ml/datasets/adult}{\emph{Adult}} dataset. This is a low dimensional, binary classification problem where the inputs are a variety of demographic data about an individual and the task is to predict whether the individual makes more or less than $\$50,000$/year.  
     
    \item \href{https://archive.ics.uci.edu/ml/datasets/statlog+(german+credit+data)}{\emph{German}} dataset. This is another low dimensional, binary classification problem where the inputs are demographic and financial information about an individual and the task is to categorize them as low or high risk for defaulting on a loan.
    
\end{itemize}

\emph{Data Transforms and Augmentation.} We used the following pre-processing operations for the different datasets: 
\begin{itemize}
    \item \emph{UTKFace.} For this dataset, we first resized the dataset to size $3 \times 50\times 50$ from the original $3\times 200 \times 200$. Then we also applied random horizontal flip with probability $0.5$. For training we used the images from the file \ttt{UTKFace.tar.gz}, while for testing we used the file \ttt{crop\_part1.tar.gz}. (Both of these files can be found at this 
    \href{https://drive.google.com/drive/folders/0BxYys69jI14kU0I1YUQyY1ZDRUE}{link} shared by the owner of the dataset).
    
    \item \emph{FashionMNIST.} We only employed the normalization transform in this case, and used the default training and test splits. 
    
    \item \emph{Cifar10.} We  employed a random crop transform~(to size $3\times 28 \times 28$ with padding $1$),  a random horizontal flip transform~(with probability $0.5$) and a normalization transform, and used the default training and test split. 
    
    \item \emph{Adult.} For the categorical inputs we used a one-hot encoding and normalized all numerical inputs to $[0,1]$.  We used the provided training/test split.
    
    \item \emph{German} We used the same pre-processing as for the {\em Adult} dataset here.  As described in the main paper, this dataset only provides 1000 examples and does not have a canonical train/test split.  Moreover, results were very sensitive to the choice of split.  So for every experiment we generated a new, random 70/30 training/test split and report results averaged over 500 such trials.  
\end{itemize}

\emph{Details of Classifiers.} We implemented the CNN using Pytorch. The CNN used for both {\em Cifar10} and {\em UTKFace} had the same architecture consisting of: 
\vspace{-1em}
\begin{itemize}\itemsep-0.2em
\item First \ttt{Conv2d} layer with $32$ output channels, kernel size $3$, padding = $1$, followed by \ttt{ReLU} followed by a \ttt{MaxPool2d} layer with kernel size $2$ and stride $2$. 

\item Second \ttt{Conv2d} layer with $64$ output channels, kernel size $3$, followed by \ttt{ReLU} followed by \ttt{MaxPool2d} with kernel size $2$. 

\item The two \ttt{Conv2d} layers were followed by $3$ fully connected layers with output sizes $600$, $120$ and $2$ respectively. 
\item  For training the neural network we used the Adam optimizer with $\ttt{lr} = 0.001$. 
\end{itemize}
The CNN used for {\em FashionMNIST} was identical, except with the second convolutional layer omitted, due to the simpler nature of the dataset.  

For both {\em Adult} and {\em German} datasets we again used the \ttt{LogisticRegression} implementation in the \ttt{scikit-learn} package.  

\emph{Computing Infrastructure used.}  The image dataset experiments were run on \href{https://colab.research.google.com}{Google Colab} using the free GPU instances and on a shared computing cluster which provides a variety of different GPUs.  So we cannot provide the exact details of the GPUs we used, but all experiments in the paper can comfortably run in less than day on a GTX 1080 TI.  

\subsection{Additional Experimental Results}
\begin{figure}[b]
\captionsetup[subfigure]{labelformat=empty}
     \centering
    \subcaptionbox{
    \label{fig:cifaropt}}{\includegraphics[width=0.43\columnwidth]{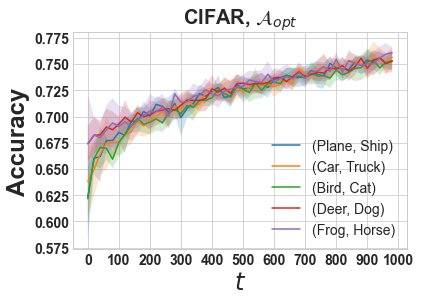}}
    \hfill   
    \subcaptionbox{
    \label{fig:cifarunif}}{\includegraphics[width=0.43\columnwidth]{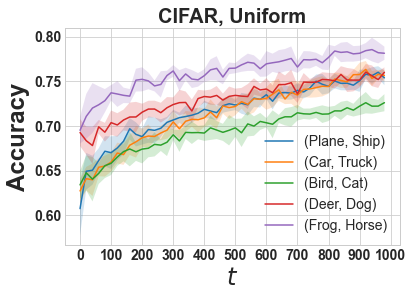}} \\ 
\vspace{-2.6em} 
\caption{Test accuracy for each attribute in {\em Cifar10} as a function of the time step, $t$, for both $\Aopt$ and \ttt{Uniform} sampling schemes, averaged over 10 trials. }
\label{fig:cifarattributes}
\vspace{-.25cm}
\end{figure}

\begin{wrapfigure}[16]{r}{.45\textwidth}
    \vspace{-1.0cm}
    \begin{center}
        \includegraphics[width=0.45\textwidth]{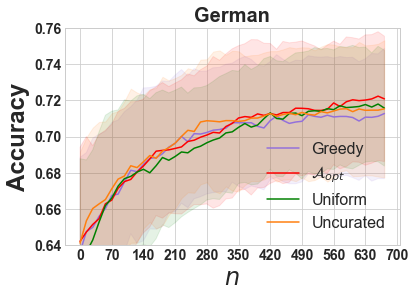}
    \end{center}
    \vspace{-.15in}
    \caption{Minimum test accuracy over all attributes for {\em German}  dataset as a function of the sampling budget $n$, averaged over 500 trials}
    \label{fig:german}
\end{wrapfigure}

Here we present results for experiments on additional datasets, these are all analogous to results presented in the main paper on other datasets.

First we have the {\em German} dataset, which is from the UCI Repository~\citep{Dua19} and qualitatively similar to the {\em Adult} dataset.  We again use a LR classifier and the exact implementation of each algorithm.  We set $\mc{Z}$ to be male or female. The results, in Figure~\ref{fig:german}, show, on average, an advantage for $\Aopt$ over both \texttt{Uniform} and \texttt{Greedy} schemes.  But this dataset contains only 1000 examples in total. We generate a 70/30 training/test split, but find that our results strongly depend on this split. To mitigate this, we run 500 trials and generate a new random split for each, but still find very high variance in our results, indicated by the shaded region in the figure.

\begin{wrapfigure}[15]{r}{.45\textwidth}
    \vspace{-0.3cm}
    \begin{center}
        \includegraphics[width=0.45\textwidth]{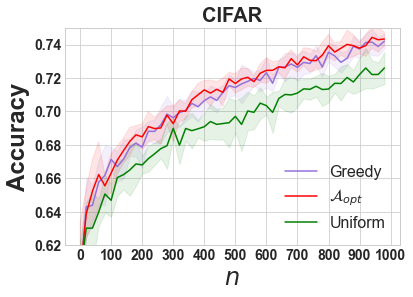}
    \end{center}
    \vspace{-.15in}
    \caption{Minimax error as a function of training round for {\em Cifar10}}
    \label{fig:cifarminimax}
\end{wrapfigure}

Figure~\ref{fig:cifarminimax} shows the minimax error on {\em Cifar10}. This again demonstrates a significant improvement in worst case accuracy for the adaptive schemes over the \texttt{Uniform} scheme, which is equivalent to \texttt{Uncurated} for this dataset.  Figures~\ref{fig:cifarattributes} and~\ref{fig:faceattributes} show accuracy across all attributes as a function of training round for {\em Cifar10} and {\em UTKFace} for both $\Aopt$ and \texttt{Uniform} schemes, as in the analogous results for {\em FashionMNIST} this demonstrates the utility of $\Aopt$ for improving fairness between groups and lends credence to our assumption of the existence of an equalizing sampling distribution for real world.

Finally Table~\ref{tab:real_datasets} summarizes the final test set accuracy, with standard deviation ranges, for each dataset and sampling scheme at the end of their respective training periods.  As in other results, the adaptive schemes show an advantage over \texttt{Uniform} in all cases, and the efficacy of the \texttt{Uncurated} scheme depends on the nature of the dataset and chosen attributes.

\begin{figure}[t]
\captionsetup[subfigure]{labelformat=empty}
     \centering
    \subcaptionbox{
    \label{fig:faceopt}}{\includegraphics[width=0.43\columnwidth]{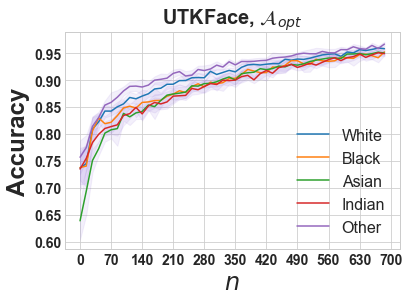}}
    \hfill   
    \subcaptionbox{
    \label{fig:faceunif}}{\includegraphics[width=0.43\columnwidth]{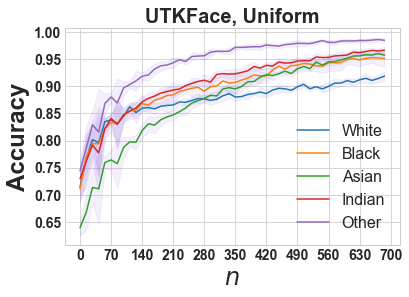}} \\ 
\vspace{-2.6em} 
\caption{Test accuracy for each attribute in {\em UTKFace} as a function of the time step, $t$, for both $\Aopt$ and \ttt{Uniform} sampling schemes, averaged over 10 trials. }
\label{fig:faceattributes}
\vspace{-.25cm}
\end{figure}

\begin{table}[ht]
    \centering
    \begin{small}
    \scalebox{0.8}{
    \begin{tabular}{|c|c|c|c|c|} \hline 
         \tbf{Dataset} & $\bm{\mc{A}_{opt}}$ & \tbf{\ttt{Greedy}} & \tbf{\ttt{Uniform}} & \tbf{\ttt{Uncurated}} \\ \hline 
         UTKFace & $0.946 \pm 0.003$ & $0.946\pm 0.012$ & $0.919\pm 0.008$ & $0.933\pm 0.007$ \\
         
         FashionMNIST & $0.936\pm 0.004$ & $0.924 \pm 0.010$ & $0.893\pm0.002$ & $0.893 \pm 0.002$ \\ 
         
         Cifar10 & $0.743\pm 0.004$ & $0.742 \pm 0.005$ & $0.726 \pm 0.010$ & $0.726 \pm 0.010$ \\ 
         
         Adult & $0.801\pm 0.001$ & $0.800 \pm 0.002$ & $0.798 \pm 0.002$ & $0.797 \pm 0.003$ \\
         
         German & $0.721\pm 0.035$ & $0.713 \pm 0.036$ & $0.716 \pm 0.032$ & $0.715 \pm 0.038$ \\
         
         \hline 
    \end{tabular}
    } 
    \end{small}
    \caption{Minimum test accuracy over different attributes achieved by the classifiers returned by the 4 sampling schemes.}
    \label{tab:real_datasets}
\end{table}

\end{appendix}

\end{document}